\documentclass{article}

\usepackage[main,final]{neurips_2026}

\usepackage[utf8]{inputenc} 
\usepackage[T1]{fontenc}    
\usepackage{hyperref}       
\usepackage{url}            
\usepackage{booktabs}       
\usepackage{amsfonts}       
\usepackage{nicefrac}       
\usepackage{microtype}      
\usepackage{xcolor}         

\usepackage{url}
\usepackage{amsthm}
\usepackage{amsmath}
\usepackage{amssymb}
\usepackage{caption}
\usepackage{hyperref}
\usepackage{enumitem}
\usepackage{mathtools}
\usepackage{adjustbox}

\theoremstyle{plain}
\newtheorem{theorem}{Theorem}
\newtheorem{proposition}[theorem]{Proposition}
\newtheorem{lemma}[theorem]{Lemma}
\newtheorem{corollary}[theorem]{Corollary}
\theoremstyle{definition}

\theoremstyle{remark}

\newcommand{\R}{\mathbb{R}}
\newcommand{\inner}[2]{\left\langle #1,\,#2 \right\rangle}

\newcommand{\softmax}{\mathrm{softmax}}
\newcommand{\Id}{I}

\usepackage{xcolor}
\definecolor{softred}{RGB}{250,100,100}
\definecolor{softgreen}{RGB}{56,118,29}
\definecolor{softblue}{RGB}{100,150,200}
\definecolor{softgray}{RGB}{150,150,150}
\definecolor{softpurple}{RGB}{145,115,185}
\definecolor{softteal}{RGB}{70,155,145}
\newcommand{\textred}[1]{{\color{softred}#1}}
\newcommand{\textblue}[1]{{\color{softblue}#1}}
\newcommand{\textgray}[1]{{\color{softgray}#1}}
\newcommand{\textgreen}[1]{{\color{softgreen}#1}}
\newcommand{\textpurple}[1]{{\color{softpurple}#1}}
\newcommand{\textteal}[1]{{\color{softteal}#1}}

\usepackage[textsize=tiny]{todonotes}

\usepackage{array}
\usepackage{multicol}
\usepackage{multirow}
\usepackage{booktabs}
\usepackage{tabularx}

\usepackage{wrapfig}
\usepackage{pgfplots}
\pgfplotsset{compat=1.18}

\title{Not All Gap Correction Helps: A Geometric View of External Information in LLM Inference}
\author{
  Dingzirui Wang\textsuperscript{1}\thanks{Work was done during a visit at SMU.},
  Xuanliang Zhang\textsuperscript{1},
  Keyan Xu\textsuperscript{1},
  Qingfu Zhu\textsuperscript{1},
  Wanxiang Che\textsuperscript{1}\thanks{Corresponding author.},
  Yang Deng\textsuperscript{2}\\
  \textsuperscript{1}Harbin Institute of Technology\quad
  \textsuperscript{2}Singapore Management University\\
  \texttt{\{dzrwang,xuanliangzhang,kyxu,qfzhu,car\}@ir.hit.edu.cn}\\
  \texttt{ydeng@smu.edu.sg}
}

\begin{document}
    \maketitle

    \begin{abstract}
        Providing external information beyond the query is a common way to improve large language model (LLM) inference, as seen in in-context learning (ICL), retrieval-augmented generation (RAG), and memory-based methods (Mem). 
        Existing work increasingly suggests that useful external information should compensate for the gap between LLMs and user queries, implicitly assuming that stronger gap-correction can lead to better performance. 
        In this paper, we challenge this assumption and show that more gap-corrective external information can instead harm performance. 
        We analyze this phenomenon in Transformers through the lens of reasoning error, defined as the difference between the predicted answer vector and the ground-truth answer vector. 
        Our analysis reveals that external information can be not only under-corrective but also over-corrective, where excessively strong correction increases reasoning error. 
        We further show that the induced correction vector is jointly determined by attention and the complementarity between the external information and the query, and derive conditions under which information is properly corrective in both direction and magnitude. 
        Experiments on four mainstream LLMs and seven reasoning benchmarks across ICL, RAG, and Mem validate our theory and support a theory-guided information selection method.
    \end{abstract}

    \section{Introduction}
        Providing query-related information in the input is an effective way to improve the inference performance of large language models (LLMs).
This paradigm underlies a wide range of methods, including In-Context Learning (ICL) with demonstrations~\cite{dong-etal-2024-survey}, Retrieval-Augmented Generation (RAG) with external knowledge~\cite{li-etal-2025-survey}, and Memory-based methods (Mem) with historical interaction~\cite{luo-etal-2026-from}.
In this paper, we refer to information provided in the input beyond the query as \emph{external information}.

Recent studies suggest that useful external information should not merely resemble the query, but should compensate for what the model lacks in solving it.
We refer to such information as \emph{gap-corrective external information}.
Existing work has recognized that external information could hurt performance when it is overly long, noisy, or redundant, and has proposed to control such issues via attention-entropy balancing, graph-based evidence organization, or sentence-level context refinement \cite{wang2025beeragbalancedentropyengineering,liu2026graphanchoredknowledgeindexingretrievalaugmented,yao-etal-2025-paretorag}\footnote{Considering the page limit, we discuss related works in Appendix~\ref{app:related_work}.}.
However, these works mainly address the quality and organization of external information while implicitly assuming that stronger gap correction is always preferable.
They therefore overlook a crucial distinction between correction that is misaligned with the model's error and correction that is directionally useful but excessive in magnitude.
In this work, we argue that not all gap correction is beneficial.
The effect of external information is fundamentally geometric: it helps only when it corrects the model in the right direction and by the proper magnitude.

To study this, we measure performance using the reasoning error, which is the vector difference between the predicted answer and the ground truth.
We analyze how external information affects this error in Transformers.
Our findings are summarized in Table~\ref{tab:findings}.
First, we show that external information can be not only \emph{under-corrective}, failing to compensate for what the model lacks, but also \emph{over-corrective}, where the model attends too strongly to the provided information, resulting in an increase in the reasoning error, as shown in Figure~\ref{fig:motivation}.
Second, we characterize the factors that determine the external \emph{information-induced correction vector}, showing that it is governed by both attention and the complementarity between the external information and the query.
Third, based on this analysis, we derive the conditions under which external information is \emph{properly corrective}, where the induced correction must satisfy a joint direction–magnitude requirement so that it is neither under-corrective nor over-corrective, ensuring to decrease the reasoning error.

For ICL, RAG, and Mem, we conduct experiments on four mainstream LLMs and seven reasoning benchmarks, empirically verifying the existence of over-corrective external information.
Further experiments show that the effects of different factors on corrective strength are consistent with our theoretical predictions and that the conditions for information to be properly corrective also align with our derivations.
In addition, based on our theoretical analysis, we propose an information selection method that maximizes the norm of the correction vector while satisfying the proper-correction condition. Compared with the baselines, it achieves an average relative error decrease of $5.9\%$, demonstrating that our theoretical findings can effectively guide practice.

Our contributions can be summarized as follows:
\begin{itemize}[leftmargin=*]
    \item We present a geometric view of how external information affects LLM inference, showing that stronger gap correction is not always better but could harm performance.
    \item We identify the key factors behind the information-induced correction vector, showing that its strength depends on both relevance, measured by attention to the external information, and complementarity, measured by its difference from the query in the value space.
    \item We establish the conditions for external information to be properly corrective, where the induced correction must have both the right direction and the proper magnitude to reduce reasoning errors.
\end{itemize}

\begin{figure}[t]
    \centering
    \small
    \includegraphics[width=0.9\linewidth]{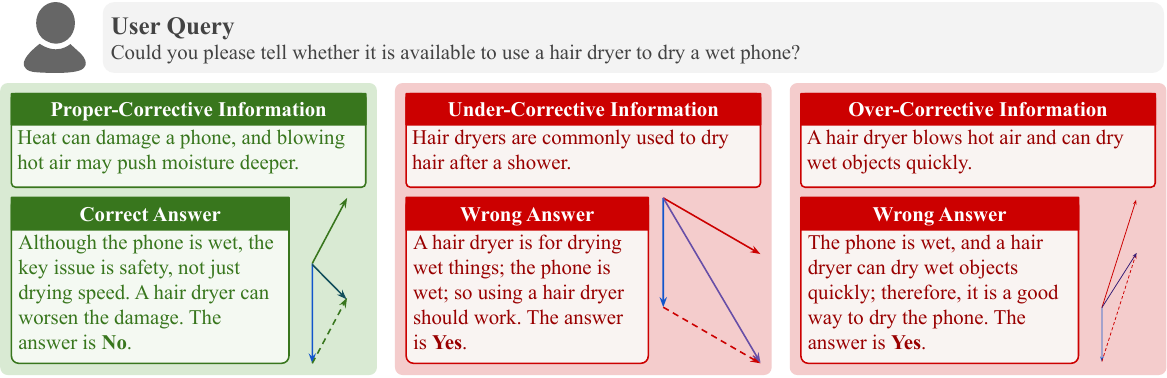}
    \caption{
        An illustration of different types of information, where \textgreen{green} denotes properly corrective information, and \textred{red} denotes under-corrective or over-corrective information.
        In the lower-right corner of each part, the \textblue{blue} vector represents the baseline error vector, while the \textgreen{green} and \textred{red} vectors represent the information-induced correction vectors.
        The \textteal{cyan} and \textpurple{purple} vectors denote the corresponding composed vectors, respectively. 
        The composed vector for properly corrective information is shorter than the baseline error vector, showing lower error.
        In contrast, the composed vectors for under-corrective or over-corrective information are longer than the baseline error vector.
    }
    \label{fig:motivation}
\end{figure}

\begin{table*}[t]
    \centering
    \small
    \caption{
        The main findings and corresponding evidence of this paper.
    }
    \begin{tabular}{>{\raggedright\arraybackslash}m{0.63\textwidth}
                >{\raggedright\arraybackslash}m{0.13\textwidth}
                >{\raggedright\arraybackslash}m{0.14\textwidth}}
    \toprule
    \textbf{Finding} & \textbf{Evidence} & \textbf{Experiment} \\
    \midrule
    Over-corrective external information could harm performance.
    & Theorem~\ref{thm:error_decomp}
    & \S\ref{subsec:experiment_harm_performance} \\
    \midrule
    Corrective strength of external information is governed by both attention and complementarity between query and information.
    & Corollary~\ref{cor:factor_affect_align}, Corollary~\ref{cor:factor_affect_align_multi_layer}
    & \S\ref{subsec:experiment_factor_affect} \\
    \midrule
    The proper-corrective external information should satisfy a joint direction-magnitude requirement $\rho(t,x) > \frac{\|g(t,x)\|}{2\,\|e(x)\|}$.
    & Corollary~\ref{cor:proper_correct}, Corollary~\ref{cor:proper_correct_multi_layer}
    & \S\ref{subsec:experiment_proper_correct} \\
    \bottomrule
\end{tabular}

    \label{tab:findings}
\end{table*}

    \section{Preliminary}
        In this section, we introduce some necessary mathematical notation.
\textbf{Notably, our simplifications of the model input--output mapping and Transformer architecture follow common practice in prior theoretical studies~\cite{zhang2024trained,lu2025asymptotic} and are further supported by the consistency between our theoretical discovery and the experimental results in \S\ref{sec:experiment}.}
We regard the external information, user query, and output as single semantic vectors after encoding them separately.
In Appendix~\ref{app:multi_token}, we discuss that our conclusions continue to hold when the inputs and outputs consist of multiple tokens.
Specifically, let $d\in\mathbb{N}$ denote the embedding dimension.
We consider the input matrix $E = [\,t,\,x\,]\in\R^{d\times 2}$, where $t\in\R^d$ represents the external information and $x\in\R^d$ represents the user query.
The prediction function $f: \R^{L \times d} \rightarrow \R^{L \times d}$ maps an input of arbitrary length $L$ to an output of the same length.
Given a target $y\in\R^d$, we view the output at the position corresponding to $x$, denoted $f_x$, as the model readout and define the \textbf{information-conditioned error vector} $e(t,x):=f_x(t,x)-y\in\R^d$.
As a comparison, we define the \textbf{baseline error vector} $e(x):= f_x(x)-y$, where $e(t,x) - e(x)$ reflects how external information affect the error.

In this paper, we primarily focus on the Transformer architecture.
In light of prior studies showing that many components of the Transformer do not materially affect the overall properties \cite{zhang2024trained,huang-2024-convergence}, we adopt a suitably simplified Transformer architecture in this paper, allowing us to focus more sharply on the fundamental properties of the model itself.
We further show in Appendix~\ref{app:other_structure_of_transformer} that the components omitted in our analysis do not affect the main conclusions of this work.
Specifically, the model we consider consists of one or more Transformer blocks, where each block contains an attention layer and an MLP layer with a residual connection.

\paragraph{Attention layer}
    Let $W_Q,W_K,W_V,W_O\in\R^{d\times d}$ be the query matrix, key matrix, value matrix, and output projection matrix, respectively.
    For any vector $z\in\R^d$, define $q_z:=W_Q z,\ k_z:=W_K z,\ v_z:=W_V z$.
    Let the score matrix $S\in\R^{2\times 2}$ have entries
    \[
    S_{uw}:=\frac{\inner{q_u}{k_w}}{\sqrt{d}},\qquad u,w\in\{t,x\}.
    \]
    Applying the $\softmax$ operation row-wise, we define the attention weights
    \[
    \alpha_{u\leftarrow w}:=\frac{\exp(S_{uw})}{\sum_{w'\in\{t,x\}}\exp(S_{uw'})},
    \qquad u,w\in\{t,x\}.
    \]
    Then the attention output at position $u\in\{t,x\}$ is defined as
    \[
    \mathrm{Attn}_u(E):=
    W_O\!\left(\sum_{w\in\{t,x\}} \alpha_{u\leftarrow w}\,v_w\right)\in\R^d.
    \]

\paragraph{MLP layer}
    Given weights $W_1,W_2\in\mathbb{R}^{d \times d}$.
    Let the nonlinear activation function $\sigma:\R^d\to\R^d$ be $1$-Lipschitz.
    We define the MLP layer with the residual connection by
    \[
    \mathrm{MLP}(z):=z+W_2\,\sigma(W_1 z),\qquad z\in\R^d.
    \]
    
Finally, the output of a single block is defined as
\[
\mathrm{Block}_u(E):=\mathrm{MLP}\bigl(E_{\bullet(u)}+\mathrm{Attn}_u(E)\bigr)\in\R^d.
\]

    \section{How External Information Affects Performance}\label{sec:analysis}
        In this part, we use the reasoning error to measure performance and progressively discuss how external information affects the reasoning error.
For ease of exposition, we first prove in the single-information, single-layer setting that there exists over-corrective external information that can increase the error, and we characterize the factors that affect the corrective effect of external information, as well as the conditions under which it is properly corrective.
We then generalize the above results to the multi-information and multi-layer setting, making the conclusions of this paper more realistic.
All proofs of this section are present in Appendix~\ref{app:proof}.

\subsection{Single Information, Single Layer}\label{subsec:single_single_analysis}
    In this part, we analyze how a single piece of external information affects the reasoning error in a single-layer Transformer.
    Let
    \[
    \delta_{\rm attn}(t,x) = \mathrm{Attn}_x([t,x])-\mathrm{Attn}_x([x])
    \]
    be the change in the attention output after introducing external information $t$.
    Let the Jacobian matrix of the MLP be
    \[
    J_{\mathrm{MLP}}(z):=\Id+W_2\,J_\sigma(W_1 z)\,W_1,
    \]
    where $\|J_\sigma(\cdot)\|\le 1$ holds almost everywhere\footnote{In this paper, $\|\cdot\|$ denotes the Euclidean norm for vectors and the operator norm induced by the Euclidean norm for matrices.}.
    Then we can derive that:
    \begin{theorem}[Correction of External Information]\label{thm:error_decomp}
        \[
        e(t,x) = e(x) + \int_{0}^{1}J_{\mathrm{MLP}}\!\bigl(\mathrm{Attn}_x([x])+s\,\delta_{\rm attn}(t,x)\bigr)\delta_{\rm attn}(t,x)\,\mathrm{d}s
        \]
    \end{theorem}
    Theorem~\ref{thm:error_decomp} gives the change in the error after introducing external information.
    We can see that the second term on the right-hand side measures how external information corrects the baseline error, so we define the \textbf{information-corrective vector} as:
    \[
    g(t,x) := \int_{0}^{1}J_{\mathrm{MLP}}\!\bigl(\mathrm{Attn}_x([x])+s\,\delta_{\rm attn}(t,x)\bigr)\delta_{\rm attn}(t,x)\,\mathrm{d}s
    \]
    As illustrated in Figure~\ref{fig:motivation}, there are two cases that lead to an increase in error:
    \textit{(i) under-corrective}: the induced correction is not directionally aligned with the baseline error, meaning that the provided information does not correct what the model lacks;
    \textit{(ii) over-corrective}: the induced correction points in the right direction, but its norm is too large, meaning that the model pays excessive attention to the provided information.
    
    Based on Theorem~\ref{thm:error_decomp}, we can conclude that:
    \begin{corollary}[Factors that Affect the Corrective Effect of External Information]\label{cor:factor_affect_align}
        For a single-layer Transformer with a single external information $t$, let $L_{\mathrm{MLP}}:=1+\|W_1\|\,\|W_2\|$.
        Then
        \[
        \|g(t,x)\| \le L_{\mathrm{MLP}} \alpha_{x\leftarrow t}\,\|W_O\|\,\|v_t-v_x\|.
        \]
    \end{corollary}
    Corollary~\ref{cor:factor_affect_align} gives the factors that affect the corrective effect of external information, including:
    \textit{(i)} \emph{relevance} $\alpha_{x \leftarrow t}$, namely how much attention the query pays to external information;
    \textit{(ii)} \emph{complementarity} $\|v_t-v_x\|$, namely how different the external information is from the query in the value projection space.
    A larger relevance or complementarity increases the strength of the induced correction, which can be helpful for decreasing error but can also push the model into the over-corrective regime when the correction becomes too large.
    
    Based on Theorem~\ref{thm:error_decomp} and Corollary~\ref{cor:factor_affect_align}, we provide the conditions under which external information is properly corrective:
    \begin{corollary}[When External Information is Properly Corrective]\label{cor:proper_correct}
        For a single-layer Transformer with a single external information $t$, if $g(t,x)\neq 0$ and $e(x)\neq 0$, define
        \[
        \rho(t,x) := \frac{\langle -e(x),\, g(t,x)\rangle}{\|e(x)\|\,\|g(t,x)\|} \in[-1,1].
        \]
        Then
        \[
        \|e(t,x)\|<\|e(x)\|
        \quad\Longleftrightarrow\quad
        \frac{\|g(t,x)\|}{2\,\|e(x)\|}<\rho(t,x).
        \]
    \end{corollary}
    In Corollary~\ref{cor:proper_correct}, $\rho(t,x)$ denotes the cosine of the angle between the correction vector $g(t,x)$ and the negative baseline error $-e(x)$.
    We refer to it as the \textbf{correction-error angle}.
    Corollary~\ref{cor:proper_correct} shows that external information can reduce the error only when the information satisfies a joint \emph{direction-magnitude} condition.
    Under such condition, under-corrective and over-corrective correspond respectively to $\rho(t,x) < 0$ and $0 < \rho(t,x) < \frac{\|g(t,x)\|}{2\,\|e(x)\|}$.

\subsection{Multiple Information, Multiple Layer}
    We finally extend the analysis to the realistic setting where a Transformer processes multiple pieces of external information across multiple layers.
    Consider a Transformer with $L$ layers.
    Let the input consist of the query position $x$ and an existing external information set $T=\{t_1,\dots,t_n\}$.
    
    Let \(h_x^{(\ell)}(T,x)\) denote the hidden state and let \(\mathrm{Attn}_x^{(\ell)}(T,x)\) denote the attention output at position \(x\) in layer \(\ell\).
    Define the reasoning error with the existing external information by
    \[
    e^{(L)}(T,x):=h_x^{(L)}(T,x)-y.
    \]
    
    Now suppose a new piece of external information $t$ is added on top of the existing set $T$.
    We suppose each residual block is $L_{\rm blk}^{(j)}$-Lipschitz.
    At layer $\ell$, its induced attention change at position $x$ is
    \[
    \delta_{\rm attn}^{(\ell)}(t;T,x)
    :=
    \mathrm{Attn}_x^{(\ell)}(T\cup\{t\},x)-\mathrm{Attn}_x^{(\ell)}(T,x).
    \]

    In analogy to the single-layer case, define the local correction vector at layer $\ell$ as
    \[
    g^{(\ell)}(t;T,x)
    :=
    \int_0^1
    J_{\mathrm{MLP}^{(\ell)}}\!\bigl(
    \mathrm{Attn}_x^{(\ell)}(T,x)+s\,\delta_{\rm attn}^{(\ell)}(t;T,x)
    \bigr)
    \,\delta_{\rm attn}^{(\ell)}(t;T,x)\,\mathrm ds.
    \]
    To propagate this local correction to the final layer, let
    \[
    \Phi^{(\ell+1:L)}(z):=f^{(L)}\circ\cdots\circ f^{(\ell+1)}(z),
    \]
    and define
    \[
    \tilde g^{(\ell)}(t;T,x)
    :=
    \int_0^1
    J_{\Phi^{(\ell+1:L)}}\!\bigl(h_x^{(\ell)}(T,x)+s\,g^{(\ell)}(t;T,x)\bigr)
    \,g^{(\ell)}(t;T,x)\,\mathrm ds.
    \]
    By the residual structure, the overall correction induced by the newly introduced external information is
    \[
    g^{(L)}(t;T,x):=\sum_{\ell=1}^L \tilde g^{(\ell)}(t;T,x),
    \]
    and therefore $e^{(L)}(t;T,x)=e^{(L)}(T,x)+g^{(L)}(t;T,x)$.

    The next two corollaries show that the same logic from the previous conclusion still holds, where the effect of new external information is governed by both its magnitude and its direction relative to the current error:
    \begin{corollary}[Factors that Affect the Corrective Effect of External Information in Multi-Layer Transformers]\label{cor:factor_affect_align_multi_layer}
        Let $L_{\mathrm{MLP}}^{(\ell)}:=1+\|W_1^{(\ell)}\|\,\|W_2^{(\ell)}\|$.
        Then the overall correction vector satisfies
        \[
        \|g^{(L)}(t;T,x)\|
        \le
        \sum_{\ell=1}^L
        \Bigl(\prod_{j=\ell+1}^L L_{\rm blk}^{(j)}\Bigr)
        L_{\mathrm{MLP}}^{(\ell)}
        \alpha_{x\leftarrow t}^{(\ell)}
        \|W_O^{(\ell)}\|
        \|\Delta v^{(\ell)}(t;T,x)\|.
        \]
    \end{corollary}
    Corollary~\ref{cor:factor_affect_align_multi_layer} shows that, at each layer, the strength of the correction induced by the new external information is still determined by two factors:
    \textit{(i)} \emph{relevance} \(\alpha_{x\leftarrow t}^{(\ell)}\), and
    \textit{(ii)} \emph{complementarity} \(\|\Delta v^{(\ell)}(t;T,x)\|\).
    Compared with the single-layer case, complementarity is now measured against the effective state already formed by the query and the existing external information, and the upper layers further scale this correction through \(\prod_{j=\ell+1}^L L_{\rm blk}^{(j)}\).
    
    \begin{corollary}[When External Information is Properly Corrective in Multi-Layer Transformers]\label{cor:proper_correct_multi_layer}
        If \(\|g^{(L)}(t;T,x)\|>0\) and \(e^{(L)}(T,x)\neq 0\), define
        \[
        \rho^{(L)}(t;T,x)
        :=
        \frac{
        \langle -e^{(L)}(T,x),\,g^{(L)}(t;T,x)\rangle
        }{
        \|e^{(L)}(T,x)\|\,\|g^{(L)}(t;T,x)\|
        }
        \in[-1,1].
        \]
        Then
        \[
        \|e^{(L)}(t;T,x)\|<\|e^{(L)}(T,x)\|
        \quad\Longleftrightarrow\quad
        \rho^{(L)}(t;T,x)>
        \frac{\|g^{(L)}(t;T,x)\|}{2\,\|e^{(L)}(T,x)\|}.
        \]
    \end{corollary}
    Corollary~\ref{cor:proper_correct_multi_layer} shows that even in the multi-information, multi-layer setting, external information reduces the reasoning error only when the induced correction is properly corrective relative to the current error.
    Hence, deeper architectures do not change the basic criterion.
    New external information helps only when it points in the right direction and has the proper scale.
    Otherwise, it could still lead to under-correction or over-correction.

    \section{Experiment}\label{sec:experiment}
        \subsection{Experiment Setup}
    \paragraph{Dataset}
        We mainly conduct experiments on three types of mainstream tasks to fully validate our theoretical conclusions, which include:
        \textit{(i)} Math: MATH~\cite{hendrycks2021measuring}, CHAMP~\cite{mao-etal-2024-champ};
        \textit{(ii)} Reason: TheoremQA~\cite{chen-etal-2023-theoremqa}, MMLU-Redux~\cite{gema-etal-2025-done}, GPQA~\cite{rein2024gpqa};
        \textit{(iii)} QA: NaturalQuestions~\cite{kwiatkowski-etal-2019-natural}, FinQA~\cite{chen-etal-2021-finqa}.
        We introduce these datasets in detail in Appendix~\ref{app:dataset}.

    \paragraph{Model}
        We conduct experiments on four mainstream LLMs, including Llama-3.1-8B-Instruct (Llama-3.1-8B) \cite{grattafiori2024llama3herdmodels}, DeepSeek-R1-Distill-Llama-8B (Llama-R1-8B) \cite{deepseekai2025deepseekr1incentivizingreasoningcapability}, Qwen3-8B \cite{yang2025qwen3technicalreport}, and Ministral-8B-Instruct-2410 (Ministral-8B) \cite{mistral2024ministral}.
        Given that these LLMs cover different families, we believe they can adequately validate our theoretical conclusions.

    \paragraph{Setup}
        We first filter the bad cases under the no-information setting, ensuring that the experimental data consist of queries that the models cannot solve on their own.
        To reflect mainstream ways of leveraging external information, we construct external information as follows:
        \textit{(i)} ICL: we use demonstrations from the training set as information;
        \textit{(ii)} RAG: we use the knowledge provided in the dataset as information;
        \textit{(iii)} Mem: at test time, we use historical test instances as information.
        The detailed methodological design and experimental configurations are provided in Appendix~\ref{app:experiment_setup}, and the practical computation of error and correction vectors is detailed in Appendix~\ref{app:calculation}.
        We merge the data from all datasets, models, and methods as our experimental data.
        To uniformly reflect the relationship, we normalize the results of all settings.
        All our experiments are run on two H100-80G and are based on Transformers~\cite{wolf-etal-2020-transformers} and vLLM~\cite{kwon2023efficient}, which cost about $20$ minutes under each setting.

\begin{figure*}[t]
    \centering
    \small

    \begin{minipage}[b]{0.49\textwidth}
        \centering
        \includegraphics[width=\linewidth]{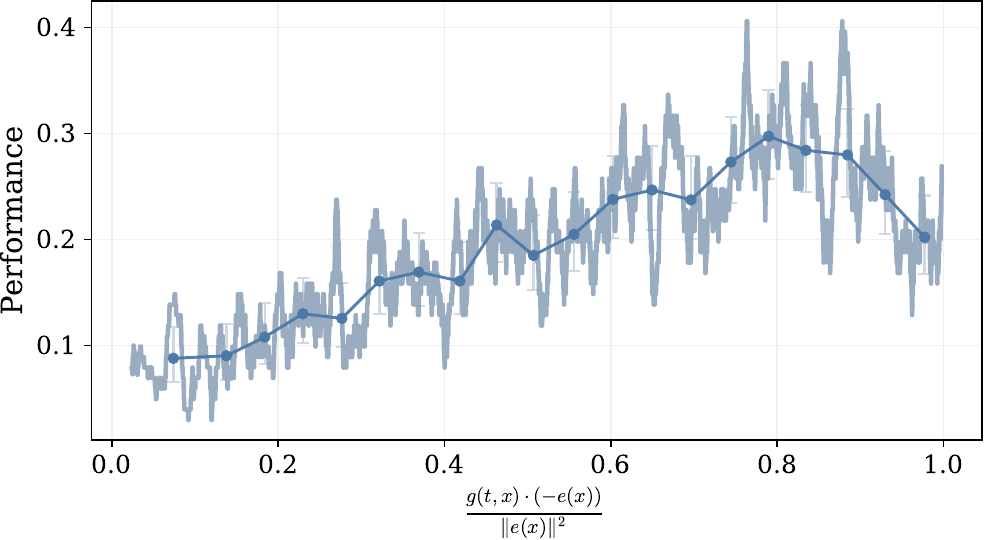}
        \captionof{figure}{
            The performance change under different relative corrective effect when using contextual information.
            Since performance of each data is a boolean value, we plot it against the first variable by grouping samples into quantile-based bins, computing the empirical accuracy and Wilson confidence interval within each bin, and overlaying a rolling-average curve to visualize the trend.
        }
        \label{fig:performance_with_correct_strength}
    \end{minipage}
    \hfill
    \begin{minipage}[b]{0.49\textwidth}
        \centering
        \includegraphics[width=\linewidth]{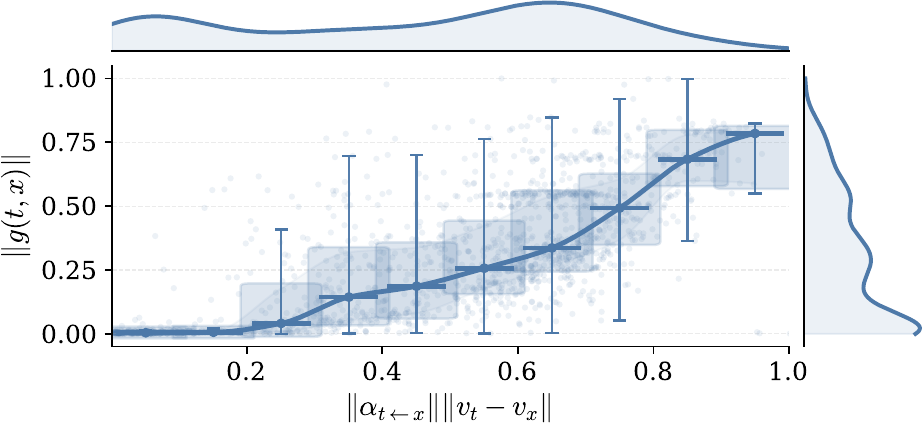}
        \captionof{figure}{
            The variation of information-corrective vector norm $\|g(t,x)\|$ with respect to the relevance and complementarity $\alpha_{x \leftarrow t}\|v_t-v_x\|$, where each point represents one data instance.
            The curves at the top and to the right show the distributions of the data points along the $x$-axis and $y$-axis, respectively.
            The Pearson correlation coefficient for the fitted points is $0.770$.
        }
        \label{fig:2d_line_g_bound}
    \end{minipage}
\end{figure*}

\subsection{RQ1: Could More Gap-Corrective External Information Harm Performance?}\label{subsec:experiment_harm_performance}
    To verify that more gap-corrective external information could harm performance, we analyze performance under different levels of relative corrective effect.
    As discussed in \S\ref{subsec:single_single_analysis}, we use $\frac{g(t,x) \cdot (-e(x))}{\|e(x)\|^2}$ to measure the relative corrective effect, i.e., the projection of the information-corrective vector onto the negative direction of the baseline error, normalized by the magnitude of the baseline error.
    We then further normalize it so that the results can be examined consistently across different datasets and models.
    The experimental results are shown in Figure~\ref{fig:performance_with_correct_strength}. 
    We observe that:
    \textit{(i)} Overall, performance first increases and then decreases as the corrective effect grows, indicating that an excessively strong corrective effect can indeed hurt performance.
    \textit{(ii)} The turning point is around $0.8$. Moreover, compared with over-correction, performance degrades more noticeably under under-correction, suggesting that for the datasets and models in our experiments, under-correction remains the primary performance bottleneck. We will further investigate this issue in \S\ref{subsec:distribution_pattern}.

\subsection{RQ2: How Different Factors Affect Correction Strength?}\label{subsec:experiment_factor_affect}
    In this part, we adapt experiments to validate the conclusion of Corollary~\ref{cor:factor_affect_align} on the factors that govern the \emph{corrective strength} of external information.
    Recall that the norm of the information-corrective vector is upper-bounded by terms jointly determined by \emph{relevance} ($\alpha_{x \leftarrow t}$) and \emph{complementarity} ($\|v_t-v_x\|$).
    Accordingly, we examine how the corrective strength $\|g(t,x)\|$ varies with these two factors.
    The experimental results of the change curve are shown in Figure~\ref{fig:2d_line_g_bound}.

    From the figure, we observe that:
    \textit{(i)} The norm of the information-corrective vector is overall positively correlated with $\alpha_{x \leftarrow t}\|v_t-v_x\|$, which supports our theoretical claim that stronger relevance and complementarity generally lead to stronger correction.
    \textit{(ii)} This relationship is not strictly linear.
    This is consistent with Corollary~\ref{cor:factor_affect_align}, which characterizes $\alpha_{x \leftarrow t}\|v_t-v_x\|$ as governing an \emph{upper bound} on the corrective strength rather than its exact value.
    As a result, many examples exhibit substantially weaker correction than the bound would permit.
    \textit{(iii)} Most samples are concentrated in the upper half of the relevance--complementarity plane, while their corrective strength remains mostly below $0.1$.
    This suggests that even when the provided external information is relatively relevant and complementary, the induced correction is often still weak.
    In other words, many candidates are potentially \emph{properly corrective}, but not sufficiently strong to fully offset the baseline error, which further explains why external information often yields only limited gains in practice.

\subsection{RQ3: When is External Information Properly Corrective?}\label{subsec:experiment_proper_correct}
    \begin{wrapfigure}{r}{0.45\textwidth}
        \centering
        \vspace{-0.5\baselineskip}
        \small
        \includegraphics[width=0.4\textwidth]{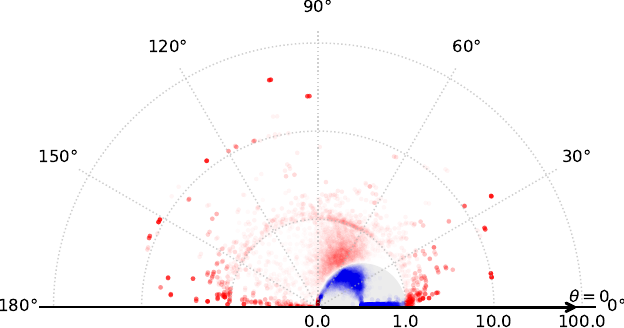}
        \caption{
            Error change with respect to the information-corrective vector $g(t, x)$, where each point in the figure represents one instance.
            The length from the axis origin to each point represents $\frac{\|g(t,x)\|}{2\|e(x)\|}$, and the angle with the positive $x$-axis represents $\arccos{\rho(t,x)}$.
            \textblue{Blue} and \textred{red} indicate that the error change is smaller or larger than $0$, and darker colors indicate larger absolute values.
            The \textgray{gray region} corresponds to $\frac{\|g(t,x)\|}{2\|e(x)\|} < \rho(t,x)$.
        }
        \label{fig:2d_polar}
        \vspace{-1\baselineskip}
    \end{wrapfigure}

    In this part, we empirically validate the conclusion of Corollary~\ref{cor:proper_correct} regarding whether external information is properly corrective.
    We analyze how the corrective direction and strength jointly affect error reduction.
    The results are shown in Figure~\ref{fig:2d_polar}, from which we make two observations:
    \textit{(i)} Almost all examples with decreased error fall into the region $\frac{\|g(t,x)\|}{2\|e(x)\|} < \rho(t,x)$.
    This shows that properly corrective information is sufficiently aligned with the negative baseline error, and its magnitude is not excessively large relative to that alignment.
    This observation provides direct empirical support for the error-reduction condition in Corollary~\ref{cor:proper_correct}.
    \textit{(ii)} There are also many examples of increased error around $\rho(t,x)=1$ and $\frac{\|g(t,x)\|}{2\|e(x)\|}=1$.
    This indicates that even when the corrective direction is nearly ideal, an excessively large corrective magnitude can still harm performance.
    In other words, these cases are \emph{over-corrective}, where the model over-emphasizes the contextual information and correspondingly underweights the user query, which leads to a substantial increase in error.

    \section{Theory-Guided Information Selection}\label{sec:selection}
        In this section, we investigate how the theoretical insights developed above can be translated into practical strategies for external information selection during LLM inference.
To this end, we study lightweight, theory-guided improvements on Llama-3.1-8B and Qwen3-8B, all of which introduce only minimal additional computational overhead.

\subsection{Distribution of Corrective Patterns}\label{subsec:distribution_pattern}
    \begin{figure}[t]
        \centering
        \begin{minipage}[t]{0.48\linewidth}
            \centering
            \small
            \begin{tikzpicture}
    \begin{axis}[
        ymin=0, ymax=100,
        ybar,
        width=\linewidth,
        height=0.65\linewidth,
        bar width=16pt,
        enlarge x limits=0.2,
        symbolic x coords={ICL,RAG,Mem},
        xtick=data,
        ylabel={Ratio},
        point meta=y,
        nodes near coords={
            \pgfmathprintnumber[fixed,fixed zerofill,precision=1]{\pgfplotspointmeta}\%
        },
        nodes near coords align={vertical},
        nodes near coords style={yshift=1.5pt, font=\tiny},
        legend style={
            at={(0.5,-0.2)},
            anchor=north,
            draw=none,
            fill=none,
            legend columns=-1,
            font=\small
        }
    ]
        \addplot+[ybar, bar shift=-9pt] coordinates {
            (ICL,88.0) (RAG,87.8) (Mem,89.5)
        };

        \addplot+[ybar, bar shift=9pt] coordinates {
            (ICL,12.0) (RAG,12.2) (Mem,10.5)
        };

        \legend{Under-Corrective, Over-Corrective}
    \end{axis}
\end{tikzpicture}
            \captionof{figure}{
                Distribution of harmful correction where the information-conditioned error vector is larger than the baseline error vector.
                \textblue{Under-Corrective} indicates that the corrective direction is misaligned with the negative baseline error.
                \textred{Over-Corrective} indicates that the corrective magnitude is excessively large.
            }
            \label{fig:error_count}
        \end{minipage}
        \hfill
        \begin{minipage}[t]{0.48\linewidth}
            \centering
            \small
            \includegraphics[width=\linewidth]{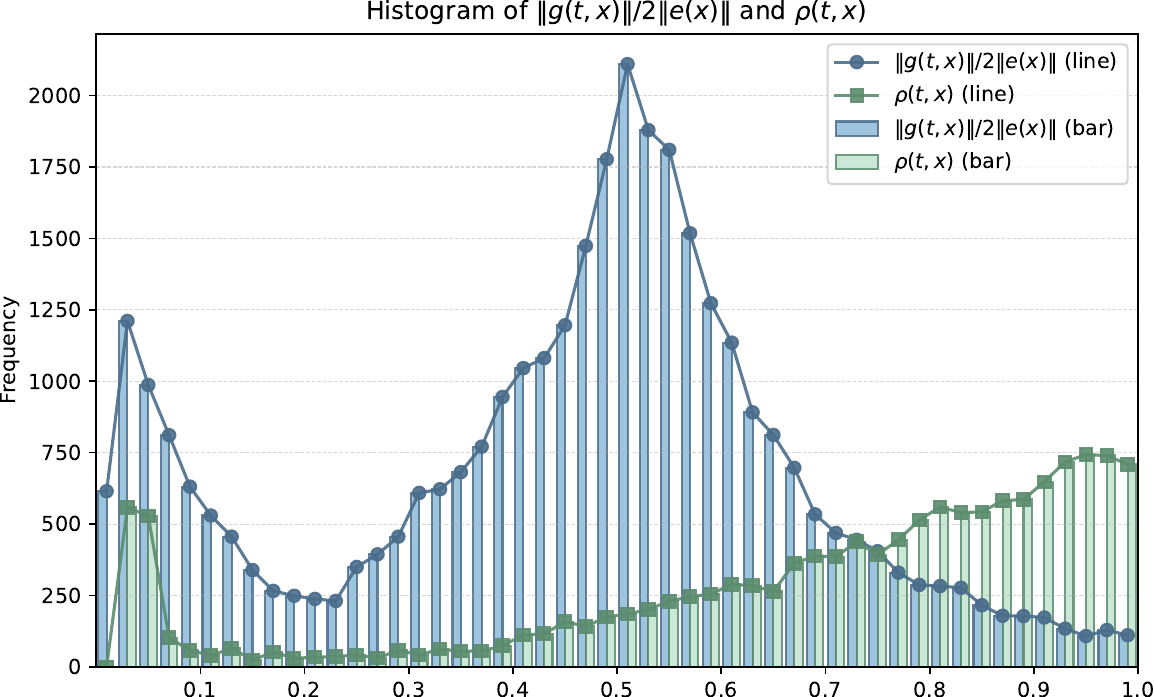}
            \captionof{figure}{
                Distribution of the corrective strength ratio $\frac{\|g(t,x)\|}{2\|e(x)\|}$ and the alignment score $\rho(t,x)$ on examples where the error decreases after using external information cross all experimental settings.
                The \textblue{blue} part denotes the corrective strength and the \textgreen{green} part denotes the angle.
            }
            \label{fig:norm_angle_distribution}
        \end{minipage}
        \vspace{-1.2em}
    \end{figure}

    To understand why external information sometimes hurts performance, we analyze examples with
    $\|e(t,x)\| > \|e(x)\|$.
    According to Theorem~\ref{thm:error_decomp}, such harmful cases mainly fall into two categories:
    \textit{(i)} \emph{over-corrective}, where $\frac{\|g(t,x)\|}{2\|e(x)\|} > \rho(t,x)$, meaning that the corrective vector is not poorly directed but is too large relative to its alignment;
    \textit{(ii)} \emph{under-corrective}, where $\rho(t,x) < 0$, meaning that the corrective vector points away from the negative direction of the baseline error.
    
    The results are shown in Figure~\ref{fig:error_count}.
    We observe that the dominant failure mode is \emph{under-correction}.
    This indicates that the main challenge in using external information is not simply to increase the correction magnitude, but to ensure that the induced correction moves the model in the proper corrective direction.
    In other words, external information may fail not because it is too weak, but because it does not sufficiently align with the direction needed to reduce the baseline error.

    We further examine examples where the error decreases after using external information.
    Figure~\ref{fig:norm_angle_distribution} shows that $\rho(t,x)$ is strongly skewed toward $1$, suggesting that helpful external information is usually well aligned with the negative baseline error.
    However, most values of $\frac{\|g(t,x)\|}{2\|e(x)\|}$ lie between $0.1$ and $0.5$.
    This indicates that many successful cases are only \emph{weakly corrective}: they point in the right direction and thus reduce the error, but their corrective strength is still not large enough to fully cancel the original baseline error.
    Therefore, successful external information is characterized by two aspects: proper directional alignment and sufficient corrective strength.

    These observations suggest that reducing reasoning error requires a two-stage use of external information.
    \textit{(i)} First, one should prioritize \emph{properly corrective} information, namely information whose corrective vector is aligned with the negative baseline error and avoids both over-correction and under-correction.
    \textit{(ii)} Then, among such candidates, one should prefer those with larger corrective strength, so that the residual error can be reduced more substantially.

\subsection{Theory-Guided External Information Selection}
    \begin{table}[t]
        \centering
        \begin{minipage}[t]{0.46\linewidth}
            \centering
            \small
            \captionof{table}{
                Performance of ICL with top-$1$ selection of five runs.
                Zero denotes inference without external information.
                The best result under each setting is marked in \textbf{bold}.
            }
            \resizebox{\linewidth}{!}{
\setlength{\tabcolsep}{3pt}
\begin{tabular}{cl|ccc}
    \toprule
    \textbf{Model} & \textbf{Method} & \textbf{MATH} & \textbf{GPQA} & \textbf{FinQA} \\
    \midrule
    \multirow{6}{*}{\rotatebox{90}{Llama-3.1-8B}}
    & Zero   & $46.0$ & $31.3$ & $48.6$ \\
    & BM25   & $46.8$ & $42.4$ & $49.2$ \\
    & DICL~\cite{kapuriya2025dicl}   & $47.4$ & $43.0$ & $51.3$ \\
    & GenICL~\cite{zhang-etal-2025-learning-select} & $47.6$ & $43.0$ & $51.6$ \\
    & TopicK~\cite{kweon-etal-2025-topic} & $47.0$ & $42.9$ & $49.8$ \\
    & \textbf{Ours}   & $\mathbf{49.6 \pm 0.7}$ & $\mathbf{46.5 \pm 2.8}$ & $\mathbf{53.1 \pm 1.2}$ \\
    \midrule
    \multirow{6}{*}{\rotatebox{90}{Qwen3-8B}}
    & Zero   & $87.0$ & $64.6$ & $60.1$ \\
    & BM25   & $88.0$ & $63.6$ & $64.2$ \\
    & DICL~\cite{kapuriya2025dicl}   & $88.0$ & $64.6$ & $67.1$ \\
    & GenICL~\cite{zhang-etal-2025-learning-select} & $88.4$ & $65.2$ & $67.3$ \\
    & TopicK~\cite{kweon-etal-2025-topic} & $86.2$ & $64.1$ & $64.2$ \\
    & \textbf{Ours}   & $\mathbf{90.2 \pm 0.5}$ & $\mathbf{65.7 \pm 2.2}$ & $\mathbf{69.5 \pm 1.7}$ \\
    \bottomrule
\end{tabular}
}
            \label{tab:performance_icl}
        \end{minipage}
        \hspace{1em}
        \begin{minipage}[t]{0.49\linewidth}
            \centering
            \small
            \captionof{table}{
                Performance of RAG with top-$1$ selection of five runs.
                TQA denotes TheoremQA and NQ denotes NaturalQuestions.
                InstructRAG refers to the InstructRAG-ICL~\cite{wei2025instructrag}.
            }
            \resizebox{\linewidth}{!}{
\setlength{\tabcolsep}{3pt}
\begin{tabular}{cl|ccc}
    \toprule
    \textbf{Model} & \textbf{Method} & \textbf{CHAMP} & \textbf{TQA} & \textbf{NQ} \\
    \midrule
    \multirow{6}{*}{\rotatebox{90}{Llama-3.1-8B}}
    & Zero & $14.1$ & $26.8$ & $38.5$ \\
    & BM25 & $22.2$ & $37.6$ & $74.7$ \\
    & DPR~\cite{karpukhin-etal-2020-dense} & $23.0$ & $38.4$ & $75.6$ \\
    & InstructRAG~\cite{wei2025instructrag} & $24.1$ & $39.2$ & $76.3$ \\
    & ReFeedL~\cite{yun2025refeed} & $24.8$ & $40.0$ & $77.1$ \\
    & \textbf{Ours} & $\mathbf{27.4 \pm 1.0}$ & $\mathbf{42.5 \pm 0.5}$ & $\mathbf{79.4 \pm 0.5}$ \\
    \midrule
    \multirow{6}{*}{\rotatebox{90}{Qwen3-8B}}
    & Zero & $39.6$ & $54.6$ & $44.2$ \\
    & BM25 & $42.6$ & $62.8$ & $73.8$ \\
    & DPR~\cite{karpukhin-etal-2020-dense} & $43.7$ & $63.6$ & $74.5$ \\
    & InstructRAG~\cite{wei2025instructrag} & $44.8$ & $64.4$ & $75.1$ \\
    & ReFeedL~\cite{yun2025refeed} & $45.2$ & $65.1$ & $75.6$ \\
    & \textbf{Ours} & $\mathbf{47.8 \pm 1.4}$ & $\mathbf{67.2 \pm 0.4}$ & $\mathbf{78.9 \pm 0.3}$ \\
    \bottomrule
\end{tabular}
}
            \label{tab:performance_rag}
        \end{minipage}
        \vspace{-1em}
    \end{table}
    \begin{table*}[t]
        \centering
        \small
        \caption{
            Performance of Mem with top-$1$ selection of five runs.
            Zero denotes inference without information.
            TQA denotes TheoremQA, MMLU denotes MMLU-Redux, and NQ denotes NaturalQuestions.
            The best performance of each setting is marked in \textbf{bold}.
        }
        \resizebox{\linewidth}{!}{
\begin{tabular}{ll|ccccccc}
    \toprule
    \textbf{Model} & \textbf{Method} & \textbf{MATH} & \textbf{CHAMP} & \textbf{TQA} & \textbf{MMLU} & \textbf{GPQA} & \textbf{NQ} & \textbf{FinQA} \\
    \midrule
    \multirow{6}{*}{Llama-3.1-8B}
    & Zero   & $46.0$ & $14.1$ & $26.8$ & $63.8$ & $31.3$ & $38.5$ & $48.6$ \\
    & BM25   & $46.4$ & $14.4$ & $27.0$ & $64.1$ & $31.8$ & $38.9$ & $49.0$ \\
    & DC~\cite{suzgun2025_DynamicCheatsheet}     & $47.0$ & $14.8$ & $27.6$ & $64.6$ & $32.3$ & $39.3$ & $49.5$ \\
    & ExpRAG~\cite{wei2025evomemorybenchmarkingllmagent} & $47.0$ & $14.8$ & $27.6$ & $64.7$ & $31.8$ & $39.4$ & $49.7$ \\
    & ReMem~\cite{wei2025evomemorybenchmarkingllmagent}  & $47.2$ & $15.2$ & $27.7$ & $64.9$ & $32.3$ & $39.8$ & $49.6$ \\
    & \textbf{Ours}   & $\mathbf{49.6 \pm 0.7}$ & $\mathbf{18.1 \pm 2.3}$ & $\mathbf{30.4 \pm 0.5}$ & $\mathbf{67.1 \pm 0.4}$ & $\mathbf{34.8 \pm 3.1}$ & $\mathbf{41.7 \pm 0.9}$ & $\mathbf{52.1 \pm 1.5}$ \\
    \midrule
    \multirow{6}{*}{Qwen3-8B}
    & Zero   & $87.0$ & $39.6$ & $54.6$ & $65.8$ & $64.6$ & $44.2$ & $60.1$ \\
    & BM25   & $87.2$ & $40.0$ & $54.8$ & $66.0$ & $65.2$ & $44.6$ & $60.4$ \\
    & DC~\cite{suzgun2025_DynamicCheatsheet}     & $88.0$ & $40.4$ & $55.4$ & $66.5$ & $65.2$ & $45.0$ & $60.9$ \\
    & ExpRAG~\cite{wei2025evomemorybenchmarkingllmagent} & $88.0$ & $40.7$ & $55.6$ & $66.6$ & $65.7$ & $45.1$ & $61.1$ \\
    & ReMem~\cite{wei2025evomemorybenchmarkingllmagent}  & $88.2$ & $40.7$ & $55.6$ & $66.8$ & $65.2$ & $45.3$ & $61.5$ \\
    & \textbf{Ours}   & $\mathbf{89.0 \pm 0.4}$ & $\mathbf{43.0 \pm 2.4}$ & $\mathbf{58.7 \pm 0.5}$ & $\mathbf{69.9 \pm 0.4}$ & $\mathbf{68.7 \pm 2.9}$ & $\mathbf{47.5 \pm 0.9}$ & $\mathbf{63.4 \pm 1.5}$ \\
    \bottomrule
\end{tabular}
}
        \label{tab:performance_me}
    \end{table*}

    Motivated by the above analysis of external information, we design a simple two-stage strategy for external information selection:
    first identify candidates that are likely to be \emph{properly corrective}, and then choose the one with the largest potential corrective strength.

    \paragraph{Corrective Direction}
        Since computing the alignment correction-error angle $\rho(t,x)$ requires the ground-truth answer, it is unavailable at inference time.
        Following prior work \cite{xie-etal-2023-doremi,liu-etal-2025-cola}, we train a smaller model to predict this quantity.
        Specifically, we use ELECTRA-Large-Discriminator~\cite{Clark2020ELECTRA} as the predictor, where the relative error is shown in Appendix~\ref{app:performance_of_trained_predictor}.
        We merge the training splits of each dataset as supervision.
        The input consists of the user query, the candidate external information, and the prediction of the LLM without external information, while the output is $\rho(t,x)$.

    \paragraph{Corrective Strength}
        Corollary~\ref{cor:factor_affect_align} shows that the norm of the corrective vector is upper-bounded by
        $\alpha_{x\leftarrow t}\|v_t-v_x\|$ up to model-dependent constants.
        This suggests that the corrective strength of external information is jointly determined by its \emph{relevance} $\alpha_{x\leftarrow t}$ and \emph{complementarity} $\|v_t-v_x\|$.
        Since $\alpha_{x\leftarrow t}$, $v_t$, and $v_x$ can be directly extracted from the model, we use
        $\alpha_{x\leftarrow t}\|v_t-v_x\|$
        as a lightweight proxy for the potential corrective strength of each candidate.

    To reduce computational cost, for each query, we first use the retrieval procedure in Appendix~\ref{app:experiment_setup} to coarsely collect $32$ candidate pieces of external information.
    We then estimate $\rho(t,x)$ for each candidate and compute its corrective-strength proxy
    $\alpha_{x\leftarrow t}\|v_t-v_x\|$.
    We first filter out candidates that are unlikely to be properly corrective according to the predicted angle and then select the candidate with the largest corrective-strength proxy among the remaining ones to guide inference.

    We evaluate this strategy under multiple settings in Table~\ref{tab:performance_icl}, Table~\ref{tab:performance_rag}, and Table~\ref{tab:performance_me}, where all methods use top-$1$ external information during inference.
    We also discuss how different factors affect the performance in Appendix~\ref{app:method_performance_factors}.
    Compared with the corresponding baselines, our method achieves comparable or better performance across settings, yielding an average relative error reduction of $5.9\%$.
    These results suggest that our theory is not only descriptive but also practically useful: it can guide external information selection by distinguishing non-corrective, over-corrective, and properly corrective candidates in a principled manner.
    Besides, the goal of this section is not to propose a new state-of-the-art system but to verify the predictive power of our theory and show that it can support principled design choices in realistic settings.


    \section{Conclusion}
        In this paper, we study how external information affects LLM reasoning in Transformers. 
        We show that stronger gap correction does not necessarily improve performance.
        Besides under-corrective information, there also exists an overlooked over-corrective regime where the model over-attends to external information, leading to an increase in reasoning errors.
        To explain this phenomenon, we characterize the correction vector induced by external information and show that its strength depends on both attention-based relevance and value-space complementarity with the query. 
        Based on this analysis, we establish the conditions for information to be properly corrective, requiring both an appropriate direction and a proper magnitude. 
        Experiments across ICL, RAG, and Mem settings on multiple LLMs and benchmarks empirically verify our theoretical findings. 
        We further show that these insights can guide practice by improving information selection. 
        Overall, our results suggest that effective context engineering should not pursue stronger correction blindly, but instead seek external information that is properly corrective to the missing capability of LLMs for the query.

    \section*{Acknowledgment}
        We gratefully acknowledge the support of the National Natural Science Foundation of China (NSFC) via grant 62236004 and 62476073.
        This research/project is also supported by the National Research Foundation Singapore under the AI Singapore Programme (AISG Award No: AISG3-RPGV-2025-016). 
        Yang Deng is support by the Lee Kong Chian Fellowship awarded by Singapore Management University.

    \bibliography{neurips_2026}
    \bibliographystyle{plain}

    \clearpage
    \appendix
    \section{Limitation}\label{app:limitation}
    This work focuses on understanding how external information affects LLM inference from the perspective of reasoning error. 
    To make the analysis tractable, we adopt a simplified Transformer formulation and represent the query, external information, and output as semantic vectors. 
    While we further extend the analysis to multi-information and multi-layer settings and empirically validate the conclusions across multiple models and benchmarks, the formulation is still an abstraction of practical LLM inference. 
    Future work may study more fine-grained settings, such as token-level dynamics, longer contexts, and multi-modal external information.

\section{Social Impact}\label{app:social_impact}
    External information is widely used in LLM applications such as retrieval-augmented generation, in-context learning, and memory-based assistants. 
    By showing that external information should be properly corrective rather than merely relevant or strong, this work may help improve the reliability of such systems and reduce cases where noisy, misleading, or over-emphasized context degrades model reasoning. 
    This can be beneficial for applications that require careful use of retrieved or historical information.
    
    At the same time, our findings also highlight that external information must be selected and used responsibly. 
    If the candidate information itself contains biased, private, outdated, or misleading content, a better selection mechanism may still amplify undesirable effects. 
    Therefore, the proposed perspective should be combined with existing safeguards for data quality, privacy protection, and bias mitigation when deployed in real-world systems.

\section{Related Work}\label{app:related_work}
    External information is a standard way to extend LLM inference beyond the knowledge stored in model parameters \cite{li-etal-2025-survey,dong-etal-2024-survey,mei2025surveycontextengineeringlarge}. 
Early work mainly follows a query-centric view, where useful external information is assumed to be the most similar or relevant to the input query. 
Representative methods therefore focus on improving query-information matching, for example by retrieving semantically similar demonstrations or passages, or by further promoting diversity and topic coverage among the selected candidates~\cite{liu-etal-2022-makes, karpukhin-etal-2020-dense, lewis-etal-2020-retrieval, shi-etal-2024-replug, kweon-etal-2025-topic}. 
More recent work questions this assumption by noting that LLMs already encode internal knowledge, so external retrieval is not always necessary. 
This has led to adaptive or model-aware methods that aim to provide external information only when the model lacks the knowledge required by the current query. 
For example, FLARE~\cite{jiang-etal-2023-active} triggers retrieval when generation becomes uncertain, Self-RAG~\cite{asai2024selfrag} learns to decide when retrieval is needed and how to critique retrieved evidence, and SKR~\cite{wang-etal-2023-self-knowledge} explicitly models what the LLM knows and does not know. 
RetrievalQA~\cite{zhang-etal-2024-retrievalqa} further benchmarks this setting, while recent analyzes show that retrieval is especially helpful for long-tail or weakly memorized knowledge. 
Along the same line, ADAPT-LLM~\cite{labruna-etal-2025-retrieve} and SeaKR~\cite{yao-etal-2025-seakr} learn retrieval triggers based on retrieval necessity or self-aware uncertainty. 
Similarly, GenICL~\cite{zhang-etal-2025-learning-select} moves beyond raw similarity by directly optimizing demonstration usefulness.

However, in this paper, we argue that when selecting external information, information that better fills the gap between the model and the query is not always better.
We reveal an \emph{over-correction} phenomenon, where the model over-attends to the provided external information, leading to an increase in reasoning errors as a result.
Based on this, we analyze the factors that govern corrective strength and derive the conditions that external information must satisfy in order to be properly corrective.

\section{Proof}\label{app:proof}
    \subsection{Supporting Lemmas}
    We first present several lemmas used in the proofs.
    
    \begin{lemma}\label{lem:attn-diff}
        For a single external information $t$,
        \[
        \delta_{\rm attn}(t,x)
        =
        \alpha_{x\leftarrow t}\,W_O\,(v_t-v_x).
        \]
    \end{lemma}
    
    \begin{proof}[Proof of Lemma~\ref{lem:attn-diff}]
        By definition,
        \[
        \mathrm{Attn}_x([t,x])
        =
        W_O\bigl(\alpha_{x\leftarrow t}v_t+\alpha_{x\leftarrow x}v_x\bigr).
        \]
        When the input only contains the query token $x$, the attention weight on $x$ is $1$, and hence
        \[
        \mathrm{Attn}_x([x])=W_Ov_x.
        \]
        Therefore,
        \[
        \begin{aligned}
            \delta_{\rm attn}(t,x)
            &=
            \mathrm{Attn}_x([t,x])-\mathrm{Attn}_x([x]) \\
            &=
            W_O\bigl(\alpha_{x\leftarrow t}v_t+\alpha_{x\leftarrow x}v_x-v_x\bigr) \\
            &=
            W_O\bigl(\alpha_{x\leftarrow t}v_t-(1-\alpha_{x\leftarrow x})v_x\bigr).
        \end{aligned}
        \]
        Since $\alpha_{x\leftarrow t}+\alpha_{x\leftarrow x}=1$, we obtain
        \[
        \delta_{\rm attn}(t,x)
        =
        \alpha_{x\leftarrow t}\,W_O\,(v_t-v_x).
        \]
    \end{proof}
    
    \begin{lemma}\label{lem:mlp-lipschitz}
        For any $z,z'\in\R^d$,
        \[
        \|\mathrm{MLP}(z)-\mathrm{MLP}(z')\|
        \le
        L_{\mathrm{MLP}}\|z-z'\|,
        \qquad
        L_{\mathrm{MLP}}:=1+\|W_1\|\,\|W_2\|.
        \]
    \end{lemma}
    
    \begin{proof}[Proof of Lemma~\ref{lem:mlp-lipschitz}]
        Since
        \[
        \mathrm{MLP}(z)=z+W_2\,\sigma(W_1z),
        \]
        we have
        \[
        \mathrm{MLP}(z)-\mathrm{MLP}(z')
        =
        (z-z')+W_2\bigl(\sigma(W_1z)-\sigma(W_1z')\bigr).
        \]
        Using the triangle inequality and the fact that $\sigma$ is $1$-Lipschitz,
        \[
        \begin{aligned}
            \|\mathrm{MLP}(z)-\mathrm{MLP}(z')\|
            &\le
            \|z-z'\|+\|W_2\|\,\|\sigma(W_1z)-\sigma(W_1z')\| \\
            &\le
            \|z-z'\|+\|W_2\|\,\|W_1z-W_1z'\| \\
            &\le
            \bigl(1+\|W_1\|\,\|W_2\|\bigr)\|z-z'\|.
        \end{aligned}
        \]
        Hence
        \[
        \|\mathrm{MLP}(z)-\mathrm{MLP}(z')\|
        \le
        L_{\mathrm{MLP}}\|z-z'\|.
        \]
    \end{proof}
    
    \begin{lemma}\label{lem:attn-diff-multi}
        Let $T=\{t_1,\dots,t_n\}$ be the existing external information set.
        Denote by $\alpha_{x\leftarrow u}^{(T)}$ the attention weights under input $[T,x]$, and define
        \[
        \bar v_{x,T}
        :=
        \sum_{u\in T\cup\{x\}}\alpha_{x\leftarrow u}^{(T)}\,v_u
        =
        \sum_{i=1}^n \alpha_{x\leftarrow t_i}^{(T)}\,v_{t_i}
        +
        \alpha_{x\leftarrow x}^{(T)}\,v_x.
        \]
        After adding a new external information $t$, denote by
        $\alpha_{x\leftarrow u}^{(T\cup\{t\})}$ the attention weights under input $[T\cup\{t\},x]$.
        Then
        \[
        \delta_{\rm attn}(t;T,x)
        :=
        \mathrm{Attn}_x([T\cup\{t\},x])-\mathrm{Attn}_x([T,x])
        =
        \alpha_{x\leftarrow t}^{(T\cup\{t\})}W_O\bigl(v_t-\bar v_{x,T}\bigr).
        \]
    \end{lemma}
    
    \begin{proof}[Proof of Lemma~\ref{lem:attn-diff-multi}]
        Let $s_u$ be the attention logit from position $x$ to token $u\in T\cup\{x\}$, and let
        \[
        Z:=\sum_{u\in T\cup\{x\}} \exp(s_u).
        \]
        Under input $[T,x]$, we have
        \[
        \alpha_{x\leftarrow u}^{(T)}=\frac{\exp(s_u)}{Z},
        \qquad
        \bar v_{x,T}
        =
        \sum_{u\in T\cup\{x\}}\alpha_{x\leftarrow u}^{(T)}v_u.
        \]
        After adding $t$, let its logit be $s_t$. Then under input $[T\cup\{t\},x]$,
        \[
        \alpha_{x\leftarrow t}^{(T\cup\{t\})}
        =
        \frac{\exp(s_t)}{Z+\exp(s_t)},
        \qquad
        \alpha_{x\leftarrow u}^{(T\cup\{t\})}
        =
        \frac{\exp(s_u)}{Z+\exp(s_t)}
        =
        \bigl(1-\alpha_{x\leftarrow t}^{(T\cup\{t\})}\bigr)\alpha_{x\leftarrow u}^{(T)}
        \quad
        (u\in T\cup\{x\}).
        \]
        Therefore,
        \[
        \begin{aligned}
            \mathrm{Attn}_x([T\cup\{t\},x])
            &=
            W_O\Bigl(
            \alpha_{x\leftarrow t}^{(T\cup\{t\})}v_t
            +
            \sum_{u\in T\cup\{x\}}\alpha_{x\leftarrow u}^{(T\cup\{t\})}v_u
            \Bigr) \\
            &=
            W_O\Bigl(
            \alpha_{x\leftarrow t}^{(T\cup\{t\})}v_t
            +
            \bigl(1-\alpha_{x\leftarrow t}^{(T\cup\{t\})}\bigr)\bar v_{x,T}
            \Bigr).
        \end{aligned}
        \]
        Subtracting
        \[
        \mathrm{Attn}_x([T,x])=W_O\bar v_{x,T},
        \]
        we obtain
        \[
        \begin{aligned}
            \delta_{\rm attn}(t;T,x)
            &=
            W_O\Bigl(
            \alpha_{x\leftarrow t}^{(T\cup\{t\})}v_t
            +
            \bigl(1-\alpha_{x\leftarrow t}^{(T\cup\{t\})}\bigr)\bar v_{x,T}
            -
            \bar v_{x,T}
            \Bigr) \\
            &=
            \alpha_{x\leftarrow t}^{(T\cup\{t\})}W_O\bigl(v_t-\bar v_{x,T}\bigr).
        \end{aligned}
        \]
    \end{proof}
    
    \subsection{Proof of Theorem~\ref{thm:error_decomp}}
    
    \begin{proof}[Proof of Theorem~\ref{thm:error_decomp}]
        Let
        \[
        z_0:=\mathrm{Attn}_x([x]),
        \qquad
        z_1:=\mathrm{Attn}_x([t,x]).
        \]
        Then by definition,
        \[
        z_1-z_0=\delta_{\rm attn}(t,x).
        \]
        Since
        \[
        e(t,x)-e(x)=\mathrm{MLP}(z_1)-\mathrm{MLP}(z_0),
        \]
        define the line segment
        \[
        z(s):=z_0+s(z_1-z_0)
        =
        \mathrm{Attn}_x([x])+s\,\delta_{\rm attn}(t,x),
        \qquad s\in[0,1].
        \]
        By the fundamental theorem of calculus,
        \[
        e(t,x)-e(x)
        =
        \int_0^1 \frac{\mathrm d}{\mathrm ds}\mathrm{MLP}\bigl(z(s)\bigr)\,\mathrm ds.
        \]
        Applying the chain rule gives
        \[
        \frac{\mathrm d}{\mathrm ds}\mathrm{MLP}\bigl(z(s)\bigr)
        =
        J_{\mathrm{MLP}}\bigl(z(s)\bigr)\,\delta_{\rm attn}(t,x).
        \]
        Therefore,
        \[
        \begin{aligned}
            e(t,x)-e(x)
            &=
            \int_0^1
            J_{\mathrm{MLP}}\!\bigl(\mathrm{Attn}_x([x])+s\,\delta_{\rm attn}(t,x)\bigr)
            \delta_{\rm attn}(t,x)\,\mathrm ds.
        \end{aligned}
        \]
        Hence
        \[
        e(t,x)
        =
        e(x)
        +
        \int_0^1
        J_{\mathrm{MLP}}\!\bigl(\mathrm{Attn}_x([x])+s\,\delta_{\rm attn}(t,x)\bigr)
        \delta_{\rm attn}(t,x)\,\mathrm ds,
        \]
        as claimed.
    \end{proof}
    
    \subsection{Proof of Corollary~\ref{cor:factor_affect_align}}
    
    \begin{proof}[Proof of Corollary~\ref{cor:factor_affect_align}]
        Recall that
        \[
        g(t,x)
        :=
        \int_0^1
        J_{\mathrm{MLP}}\!\bigl(\mathrm{Attn}_x([x])+s\,\delta_{\rm attn}(t,x)\bigr)
        \delta_{\rm attn}(t,x)\,\mathrm ds.
        \]
        By the triangle inequality and the operator norm inequality,
        \[
        \begin{aligned}
            \|g(t,x)\|
            &\le
            \int_0^1
            \Bigl\|
            J_{\mathrm{MLP}}\!\bigl(\mathrm{Attn}_x([x])+s\,\delta_{\rm attn}(t,x)\bigr)
            \Bigr\|
            \,
            \|\delta_{\rm attn}(t,x)\|
            \,\mathrm ds.
        \end{aligned}
        \]
        Since
        \[
        J_{\mathrm{MLP}}(z)=\Id+W_2\,J_\sigma(W_1z)\,W_1,
        \qquad
        \|J_\sigma(\cdot)\|\le 1
        \quad \text{a.e.},
        \]
        we have for almost every $z$,
        \[
        \|J_{\mathrm{MLP}}(z)\|
        \le
        1+\|W_1\|\,\|W_2\|
        =
        L_{\mathrm{MLP}}.
        \]
        Therefore,
        \[
        \|g(t,x)\|
        \le
        L_{\mathrm{MLP}}\|\delta_{\rm attn}(t,x)\|.
        \]
        By Lemma~\ref{lem:attn-diff},
        \[
        \delta_{\rm attn}(t,x)=\alpha_{x\leftarrow t}W_O(v_t-v_x),
        \]
        and thus
        \[
        \|\delta_{\rm attn}(t,x)\|
        \le
        \alpha_{x\leftarrow t}\,\|W_O\|\,\|v_t-v_x\|.
        \]
        Combining the two bounds yields
        \[
        \|g(t,x)\|
        \le
        L_{\mathrm{MLP}}
        \alpha_{x\leftarrow t}\,\|W_O\|\,\|v_t-v_x\|.
        \]
    \end{proof}
    
    \subsection{Proof of Corollary~\ref{cor:proper_correct}}
    
    \begin{proof}[Proof of Corollary~\ref{cor:proper_correct}]
        By Theorem~\ref{thm:error_decomp},
        \[
        e(t,x)=e(x)+g(t,x).
        \]
        Therefore,
        \[
        \|e(t,x)\|<\|e(x)\|
        \quad\Longleftrightarrow\quad
        \|e(x)+g(t,x)\|^2<\|e(x)\|^2.
        \]
        Expanding the square on the left-hand side gives
        \[
        \|e(x)\|^2+2\langle e(x),g(t,x)\rangle+\|g(t,x)\|^2<\|e(x)\|^2,
        \]
        which is equivalent to
        \[
        2\langle e(x),g(t,x)\rangle+\|g(t,x)\|^2<0.
        \]
        Rearranging yields
        \[
        \langle -e(x),g(t,x)\rangle>\frac{\|g(t,x)\|^2}{2}.
        \]
        Since $g(t,x)\neq 0$ and $e(x)\neq 0$, dividing both sides by
        $\|e(x)\|\,\|g(t,x)\|$ gives
        \[
        \frac{\langle -e(x),g(t,x)\rangle}{\|e(x)\|\,\|g(t,x)\|}
        >
        \frac{\|g(t,x)\|}{2\,\|e(x)\|}.
        \]
        By the definition
        \[
        \rho(t,x)
        :=
        \frac{\langle -e(x),g(t,x)\rangle}{\|e(x)\|\,\|g(t,x)\|},
        \]
        we conclude that
        \[
        \|e(t,x)\|<\|e(x)\|
        \quad\Longleftrightarrow\quad
        \frac{\|g(t,x)\|}{2\,\|e(x)\|}<\rho(t,x).
        \]
    \end{proof}

    \subsection{Proof of Corollary~\ref{cor:factor_affect_align_multi_layer}}
    
    \begin{proof}[Proof of Corollary~\ref{cor:factor_affect_align_multi_layer}]
        For each layer $\ell$, let
        \[
        v_u^{(\ell)}:=W_V^{(\ell)}h_u^{(\ell-1)},
        \qquad
        \bar v_{x,T}^{(\ell)}
        :=
        \sum_{u\in T\cup\{x\}}
        \alpha_{x\leftarrow u}^{(\ell,T)}\,v_u^{(\ell)},
        \]
        where $\alpha_{x\leftarrow u}^{(\ell,T)}$ denotes the attention weight at layer $\ell$
        under input $(T,x)$.
        Define
        \[
        \Delta v^{(\ell)}(t;T,x)
        :=
        v_t^{(\ell)}-\bar v_{x,T}^{(\ell)}.
        \]
        Then, by the same argument as Lemma~\ref{lem:attn-diff-multi}, the attention change at layer $\ell$ can be written as
        \[
        \delta_{\rm attn}^{(\ell)}(t;T,x)
        =
        \alpha_{x\leftarrow t}^{(\ell)}
        W_O^{(\ell)}
        \Delta v^{(\ell)}(t;T,x),
        \]
        where $\alpha_{x\leftarrow t}^{(\ell)}$ is the attention weight assigned to the newly added information $t$ at layer $\ell$.
    
        By definition,
        \[
        g^{(\ell)}(t;T,x)
        =
        \int_0^1
        J_{\mathrm{MLP}^{(\ell)}}\!\bigl(
        \mathrm{Attn}_x^{(\ell)}(T,x)+s\,\delta_{\rm attn}^{(\ell)}(t;T,x)
        \bigr)
        \delta_{\rm attn}^{(\ell)}(t;T,x)\,\mathrm ds.
        \]
        Since
        \[
        J_{\mathrm{MLP}^{(\ell)}}(z)
        =
        \Id+W_2^{(\ell)}J_\sigma\!\bigl(W_1^{(\ell)}z\bigr)W_1^{(\ell)},
        \]
        we have, for almost every $z$,
        \[
        \bigl\|J_{\mathrm{MLP}^{(\ell)}}(z)\bigr\|
        \le
        1+\|W_1^{(\ell)}\|\,\|W_2^{(\ell)}\|
        =
        L_{\mathrm{MLP}}^{(\ell)}.
        \]
        Hence
        \[
        \|g^{(\ell)}(t;T,x)\|
        \le
        L_{\mathrm{MLP}}^{(\ell)}
        \|\delta_{\rm attn}^{(\ell)}(t;T,x)\|.
        \]
        Combining this with the above expression for $\delta_{\rm attn}^{(\ell)}(t;T,x)$ gives
        \[
        \|g^{(\ell)}(t;T,x)\|
        \le
        L_{\mathrm{MLP}}^{(\ell)}
        \alpha_{x\leftarrow t}^{(\ell)}
        \|W_O^{(\ell)}\|
        \|\Delta v^{(\ell)}(t;T,x)\|.
        \]
    
        Next, by definition,
        \[
        \tilde g^{(\ell)}(t;T,x)
        :=
        \int_0^1
        J_{\Phi^{(\ell+1:L)}}\!\bigl(
        h_x^{(\ell)}(T,x)+s\,g^{(\ell)}(t;T,x)
        \bigr)
        g^{(\ell)}(t;T,x)\,\mathrm ds.
        \]
        Since each residual block is $L_{\rm blk}^{(j)}$-Lipschitz, the composition
        \[
        \Phi^{(\ell+1:L)}=f^{(L)}\circ\cdots\circ f^{(\ell+1)}
        \]
        is
        \[
        \Bigl(\prod_{j=\ell+1}^L L_{\rm blk}^{(j)}\Bigr)\text{-Lipschitz}.
        \]
        Therefore, its Jacobian norm is bounded almost everywhere by the same constant:
        \[
        \bigl\|J_{\Phi^{(\ell+1:L)}}(z)\bigr\|
        \le
        \prod_{j=\ell+1}^L L_{\rm blk}^{(j)}
        \qquad \text{for a.e. } z.
        \]
        It follows that
        \[
        \begin{aligned}
            \|\tilde g^{(\ell)}(t;T,x)\|
            &\le
            \int_0^1
            \Bigl\|
            J_{\Phi^{(\ell+1:L)}}\!\bigl(
            h_x^{(\ell)}(T,x)+s\,g^{(\ell)}(t;T,x)
            \bigr)
            \Bigr\|
            \|g^{(\ell)}(t;T,x)\|
            \,\mathrm ds \\
            &\le
            \Bigl(\prod_{j=\ell+1}^L L_{\rm blk}^{(j)}\Bigr)
            \|g^{(\ell)}(t;T,x)\|.
        \end{aligned}
        \]
        Finally, since
        \[
        g^{(L)}(t;T,x)=\sum_{\ell=1}^L \tilde g^{(\ell)}(t;T,x),
        \]
        we obtain
        \[
        \begin{aligned}
            \|g^{(L)}(t;T,x)\|
            &\le
            \sum_{\ell=1}^L \|\tilde g^{(\ell)}(t;T,x)\| \\
            &\le
            \sum_{\ell=1}^L
            \Bigl(\prod_{j=\ell+1}^L L_{\rm blk}^{(j)}\Bigr)
            \|g^{(\ell)}(t;T,x)\| \\
            &\le
            \sum_{\ell=1}^L
            \Bigl(\prod_{j=\ell+1}^L L_{\rm blk}^{(j)}\Bigr)
            L_{\mathrm{MLP}}^{(\ell)}
            \alpha_{x\leftarrow t}^{(\ell)}
            \|W_O^{(\ell)}\|
            \|\Delta v^{(\ell)}(t;T,x)\|.
        \end{aligned}
        \]
        This is exactly the desired bound.
    \end{proof}
    
    \subsection{Proof of Corollary~\ref{cor:proper_correct_multi_layer}}
    
    \begin{proof}[Proof of Corollary~\ref{cor:proper_correct_multi_layer}]
        By construction,
        \[
        e^{(L)}(t;T,x)=e^{(L)}(T,x)+g^{(L)}(t;T,x).
        \]
        Therefore,
        \[
        \|e^{(L)}(t;T,x)\|<\|e^{(L)}(T,x)\|
        \quad\Longleftrightarrow\quad
        \|e^{(L)}(T,x)+g^{(L)}(t;T,x)\|^2<\|e^{(L)}(T,x)\|^2.
        \]
        Expanding the square gives
        \[
        \|e^{(L)}(T,x)\|^2
        +
        2\langle e^{(L)}(T,x),g^{(L)}(t;T,x)\rangle
        +
        \|g^{(L)}(t;T,x)\|^2
        <
        \|e^{(L)}(T,x)\|^2,
        \]
        which is equivalent to
        \[
        2\langle e^{(L)}(T,x),g^{(L)}(t;T,x)\rangle
        +
        \|g^{(L)}(t;T,x)\|^2
        <0.
        \]
        Rearranging yields
        \[
        \langle -e^{(L)}(T,x),g^{(L)}(t;T,x)\rangle
        >
        \frac{\|g^{(L)}(t;T,x)\|^2}{2}.
        \]
        Since $\|g^{(L)}(t;T,x)\|>0$ and $e^{(L)}(T,x)\neq 0$, dividing both sides by
        $\|e^{(L)}(T,x)\|\,\|g^{(L)}(t;T,x)\|$ gives
        \[
        \frac{
        \langle -e^{(L)}(T,x),g^{(L)}(t;T,x)\rangle
        }{
        \|e^{(L)}(T,x)\|\,\|g^{(L)}(t;T,x)\|
        }
        >
        \frac{\|g^{(L)}(t;T,x)\|}{2\,\|e^{(L)}(T,x)\|}.
        \]
        By the definition
        \[
        \rho^{(L)}(t;T,x)
        :=
        \frac{
        \langle -e^{(L)}(T,x),\,g^{(L)}(t;T,x)\rangle
        }{
        \|e^{(L)}(T,x)\|\,\|g^{(L)}(t;T,x)\|
        },
        \]
        we conclude that
        \[
        \|e^{(L)}(t;T,x)\|<\|e^{(L)}(T,x)\|
        \quad\Longleftrightarrow\quad
        \rho^{(L)}(t;T,x)>
        \frac{\|g^{(L)}(t;T,x)\|}{2\,\|e^{(L)}(T,x)\|}.
        \]
    \end{proof}

\section{Multi-Token Query and Multi-Token External Information}\label{app:multi_token}
    In the main text, we analyze the corrective effect of external information by treating both the user query and each piece of external information as a single token.
    In practice, however, a user query is typically a sequence of tokens, and each piece of external information may also contain multiple tokens.
    In this appendix, we show that the conclusions in the main text remain qualitatively unchanged in this more realistic setting.
    The main difference is that a multi-token query or a multi-token external information span should be viewed as an \emph{attention-weighted virtual token} at the target position.
    
    \subsection{Single Layer}
        Let the user query prefix before the target position be
        \[
        Q=\{q_1,\dots,q_m\},
        \]
        and let the already provided external information tokens be collected into a set
        \[
        U=\{u_1,\dots,u_n\}.
        \]
        We still focus on the target position \(x\), which denotes the position whose hidden state is used to produce the final answer.
        For a single-layer Transformer, the attention output at position \(x\) on the token set \(U\cup Q\cup\{x\}\) is
        \[
        \mathrm{Attn}_x(U,Q,x)
        =
        \sum_{z\in U\cup Q\cup\{x\}}
        \alpha_{x\leftarrow z}(U,Q,x)\,W_O v_z.
        \]
        
        Now suppose we additionally provide a new piece of external information
        \[
        \tau=\{t_1,\dots,t_k\},
        \]
        which itself is a multi-token span.
        Denote by
        \[
        \alpha_{x\leftarrow \tau}
        :=
        \sum_{j=1}^k \alpha_{x\leftarrow t_j}(U,Q,\tau,x)
        \]
        the total attention mass assigned by position \(x\) to this information span after it is introduced.
        When \(\alpha_{x\leftarrow \tau}>0\), define its attention-weighted virtual value by
        \[
        \bar v_{\tau}
        :=
        \frac{1}{\alpha_{x\leftarrow \tau}}
        \sum_{j=1}^k
        \alpha_{x\leftarrow t_j}(U,Q,\tau,x)\,v_{t_j}.
        \]
        Similarly, define the effective value state formed by the original query and previously provided information as
        \[
        \bar v_{x;U,Q}
        :=
        \sum_{z\in U\cup Q\cup\{x\}}
        \alpha_{x\leftarrow z}(U,Q,x)\,v_z.
        \]
        
        The next proposition shows that, from the viewpoint of the target position \(x\), a multi-token information span behaves exactly like a single virtual token.
    
        \begin{proposition}[Virtual-Token Reduction for a Multi-Token Information Span]\label{prop:virtual_token_single_layer}
            Assume \(\alpha_{x\leftarrow \tau}>0\).
            Then the change in the attention output caused by introducing the multi-token external information span \(\tau\) is
            \[
            \delta_{\rm attn}(\tau;U,Q,x)
            :=
            \mathrm{Attn}_x(U,Q,\tau,x)-\mathrm{Attn}_x(U,Q,x)
            =
            \alpha_{x\leftarrow \tau}\,W_O\,\bigl(\bar v_{\tau}-\bar v_{x;U,Q}\bigr).
            \]
            If \(\alpha_{x\leftarrow \tau}=0\), then \(\delta_{\rm attn}(\tau;U,Q,x)=0\).
        \end{proposition}
        \begin{proof}
            Let
            \[
            Z_{\rm old}
            =
            \sum_{z\in U\cup Q\cup\{x\}} e^{s_z},
            \qquad
            Z_{\rm new}
            =
            Z_{\rm old}+\sum_{j=1}^k e^{s_{t_j}},
            \]
            where \(s_z\) is the attention score between position \(x\) and token \(z\).
            For each old token \(z\in U\cup Q\cup\{x\}\), its score does not change after adding \(\tau\), hence
            \[
            \alpha_{x\leftarrow z}(U,Q,\tau,x)
            =
            \frac{e^{s_z}}{Z_{\rm new}}
            =
            \frac{Z_{\rm old}}{Z_{\rm new}}
            \alpha_{x\leftarrow z}(U,Q,x)
            =
            (1-\alpha_{x\leftarrow \tau})\alpha_{x\leftarrow z}(U,Q,x).
            \]
            Therefore,
            \begin{align*}
                \delta_{\rm attn}(\tau;U,Q,x)
                &=
                \sum_{j=1}^k \alpha_{x\leftarrow t_j}(U,Q,\tau,x)\,W_O v_{t_j}
                \\
                &\quad+
                \sum_{z\in U\cup Q\cup\{x\}}
                \Bigl(\alpha_{x\leftarrow z}(U,Q,\tau,x)-\alpha_{x\leftarrow z}(U,Q,x)\Bigr)\,W_O v_z
                \\
                &=
                \sum_{j=1}^k \alpha_{x\leftarrow t_j}(U,Q,\tau,x)\,W_O v_{t_j}
                -
                \alpha_{x\leftarrow \tau}
                \sum_{z\in U\cup Q\cup\{x\}}
                \alpha_{x\leftarrow z}(U,Q,x)\,W_O v_z
                \\
                &=
                \alpha_{x\leftarrow \tau}\,W_O\,\bigl(\bar v_{\tau}-\bar v_{x;U,Q}\bigr).
            \end{align*}
        \end{proof}
        
        Proposition~\ref{prop:virtual_token_single_layer} shows that the only quantities that matter at the span level are:
        \textit{(i)} the total relevance \(\alpha_{x\leftarrow \tau}\), namely the total attention mass paid to this information span; and
        \textit{(ii)} the span-level complementarity \(\|\bar v_{\tau}-\bar v_{x;U,Q}\|\), namely how different the span's virtual value is from the current effective state.
        
        We can now directly extend the error decomposition in the main text.
        
        \begin{corollary}[Correction of a Multi-Token Information Span]
            Let
            \[
            e(U,Q,x):=f_x(U,Q,x)-y,
            \qquad
            e(\tau;U,Q,x):=f_x(U,Q,\tau,x)-y.
            \]
            Then
            \[
            e(\tau;U,Q,x)
            =
            e(U,Q,x)+g(\tau;U,Q,x),
            \]
            where
            \[
            g(\tau;U,Q,x)
            :=
            \int_0^1
            J_{\mathrm{MLP}}\!\Bigl(
            \mathrm{Attn}_x(U,Q,x)+s\,\delta_{\rm attn}(\tau;U,Q,x)
            \Bigr)
            \delta_{\rm attn}(\tau;U,Q,x)\,\mathrm ds.
            \]
        \end{corollary}
        
        Hence, all conclusions in the single-token case remain valid after replacing the single-token value \(v_t\) with the virtual value \(\bar v_\tau\), and replacing the original query-side state \(v_x\) with the effective state \(\bar v_{x;U,Q}\).
        
        \begin{corollary}[Factors that Affect the Corrective Effect of a Multi-Token Information Span]\label{cor:multi_token_factor}
            Let \(L_{\mathrm{MLP}}:=1+\|W_1\|\,\|W_2\|\).
            Then
            \[
            \|g(\tau;U,Q,x)\|
            \le
            L_{\mathrm{MLP}}\,
            \alpha_{x\leftarrow \tau}\,
            \|W_O\|\,
            \|\bar v_{\tau}-\bar v_{x;U,Q}\|.
            \]
        \end{corollary}
        
        Corollary~\ref{cor:multi_token_factor} shows that increasing the number of tokens in the information span does not by itself guaranty a stronger corrective effect.
        Its effect is governed only through:
        \textit{(i)} how much total attention the span attracts, and
        \textit{(ii)} how far its attention-weighted virtual value moves away from the current effective state.
        Therefore, a long but redundant span may have little effect, whereas a short but highly attended and complementary span may induce a large correction and even lead to over-correction.
        
        \begin{corollary}[When a Multi-Token Information Span is Properly Corrective]
            If \(g(\tau;U,Q,x)\neq 0\) and \(e(U,Q,x)\neq 0\), define
            \[
            \rho(\tau;U,Q,x)
            :=
            \frac{
            \langle -e(U,Q,x),\, g(\tau;U,Q,x)\rangle
            }{
            \|e(U,Q,x)\|\,\|g(\tau;U,Q,x)\|
            }
            \in[-1,1].
            \]
            Then
            \[
            \|e(\tau;U,Q,x)\|<\|e(U,Q,x)\|
            \quad\Longleftrightarrow\quad
            \frac{\|g(\tau;U,Q,x)\|}{2\,\|e(U,Q,x)\|}
            <
            \rho(\tau;U,Q,x).
            \]
        \end{corollary}
        
        Thus, even when the newly provided external information contains multiple tokens, the same direction-magnitude criterion still characterizes whether it is properly corrective.
        Under-correction and over-correction still correspond, respectively, to poor directional alignment and excessively large correction magnitude.

    \subsection{Effect of a Multi-Token Query}
        We now specialize the above formulation to the case where there is no previously provided external information, namely \(U=\varnothing\), and the user query itself contains multiple tokens.
        In this case, the effective query state becomes
        \[
        \bar v_{x;Q}
        =
        \sum_{z\in Q\cup\{x\}}
        \alpha_{x\leftarrow z}(Q,x)\,v_z.
        \]
        Equivalently, if
        \[
        \alpha_{x\leftarrow Q}:=\sum_{q\in Q}\alpha_{x\leftarrow q}(Q,x),
        \qquad
        \bar v_Q
        :=
        \frac{1}{\alpha_{x\leftarrow Q}}
        \sum_{q\in Q}\alpha_{x\leftarrow q}(Q,x)\,v_q,
        \]
        then
        \[
        \bar v_{x;Q}
        =
        \alpha_{x\leftarrow x}v_x+\alpha_{x\leftarrow Q}\bar v_Q.
        \]
        
        Therefore, relative to the single-token analysis in the main text, a multi-token user query only changes the reference state from \(v_x\) to the effective query state \(\bar v_{x;Q}\).
        The qualitative conclusion remains unchanged:
        external information is helpful only when its induced correction points toward the missing direction of the current query state and has an appropriate scale.
        Hence, a longer query mainly affects the corrective effect by changing the current effective state against which external information is compared.
    
    \subsection{Multiple Layers}
        The same reduction also applies layer by layer in a multi-layer Transformer.
        Let the query prefix be \(Q\), let the already provided external information tokens be \(U\), and let \(\tau=\{t_1,\dots,t_k\}\) be a newly added multi-token information span.
        At layer \(\ell\), define
        \[
        \alpha_{x\leftarrow \tau}^{(\ell)}
        :=
        \sum_{j=1}^k \alpha_{x\leftarrow t_j}^{(\ell)}(U,Q,\tau,x),
        \]
        \[
        \bar v_{\tau}^{(\ell)}
        :=
        \frac{1}{\alpha_{x\leftarrow \tau}^{(\ell)}}
        \sum_{j=1}^k
        \alpha_{x\leftarrow t_j}^{(\ell)}(U,Q,\tau,x)\,v_{t_j}^{(\ell)},
        \]
        and
        \[
        \bar v_{x;U,Q}^{(\ell)}
        :=
        \sum_{z\in U\cup Q\cup\{x\}}
        \alpha_{x\leftarrow z}^{(\ell)}(U,Q,x)\,v_z^{(\ell)}.
        \]
        Then the attention change induced by introducing \(\tau\) at layer \(\ell\) satisfies
        \[
        \delta_{\rm attn}^{(\ell)}(\tau;U,Q,x)
        =
        \alpha_{x\leftarrow \tau}^{(\ell)}\,
        W_O^{(\ell)}\,
        \bigl(\bar v_{\tau}^{(\ell)}-\bar v_{x;U,Q}^{(\ell)}\bigr).
        \]
        
        Accordingly, the local correction vector at layer \(\ell\) is
        \[
        g^{(\ell)}(\tau;U,Q,x)
        :=
        \int_0^1
        J_{\mathrm{MLP}^{(\ell)}}\!\Bigl(
        \mathrm{Attn}_x^{(\ell)}(U,Q,x)
        +s\,\delta_{\rm attn}^{(\ell)}(\tau;U,Q,x)
        \Bigr)
        \delta_{\rm attn}^{(\ell)}(\tau;U,Q,x)\,\mathrm ds,
        \]
        and the overall correction at the final layer is still
        \[
        g^{(L)}(\tau;U,Q,x)
        =
        \sum_{\ell=1}^L \tilde g^{(\ell)}(\tau;U,Q,x),
        \]
        where each \(\tilde g^{(\ell)}\) is the propagated version of the local correction from layer \(\ell\) to the final layer, defined exactly as in the main text.
        
        Therefore, Corollary~\ref{cor:factor_affect_align_multi_layer} and Corollary~\ref{cor:proper_correct_multi_layer} in the main text continue to hold after replacing the single-token value difference
        \[
        \Delta v^{(\ell)}(t;T,x)=v_t^{(\ell)}-\bar v_{x,T}^{(\ell)}
        \]
        with the span-level value difference
        \[
        \Delta v^{(\ell)}(\tau;U,Q,x)
        :=
        \bar v_{\tau}^{(\ell)}-\bar v_{x;U,Q}^{(\ell)}.
        \]
        In particular, we obtain
        \[
        \|g^{(L)}(\tau;U,Q,x)\|
        \le
        \sum_{\ell=1}^L
        \Bigl(\prod_{j=\ell+1}^L L_{\rm blk}^{(j)}\Bigr)
        L_{\mathrm{MLP}}^{(\ell)}
        \alpha_{x\leftarrow \tau}^{(\ell)}
        \|W_O^{(\ell)}\|
        \|\Delta v^{(\ell)}(\tau;U,Q,x)\|.
        \]
        
        Hence, even when both the user query and the provided external information contain multiple tokens, the main conclusions of this paper are unchanged.
        The only essential modification is that the corrective effect should be measured at the level of \emph{attention-weighted token groups} rather than individual tokens.
        Under this view:
        \textit{(i)} relevance becomes the total attention mass assigned to the information span, and
        \textit{(ii)} complementarity becomes the distance between the span-level virtual value and the current effective state.
        The same direction-magnitude criterion still determines whether external information is properly corrective, under-corrective, or over-corrective.

\section{Experimental Datasets}\label{app:dataset}
    \paragraph{MATH}
        MATH is a competition-style mathematics benchmark designed to measure advanced symbolic and multi-step problem solving. It contains 12.5k problems drawn from high-school math contests and textbooks, spanning multiple subjects (e.g., algebra, geometry, number theory, counting \& probability, etc.) and difficulty levels. Each example provides a problem statement together with a full, human-written step-by-step solution and a final boxed answer, which supports evaluating both final-answer correctness and intermediate reasoning. In LLM evaluation, MATH is typically scored by exact match on the final answer (often after normalization).
        We employ the MATH500 \citep{lightman2024lets} version for our experiments. 
    
    \paragraph{CHAMP}
        CHAMP (Concept and Hint-Annotated Math Problems) is a competition-level math dataset built for fine-grained analysis of how models use contextual information. Beyond the original problem and gold solution/answer, each instance is annotated with (1) concepts—general mathematical facts or principles that may be relevant—and (2) hints—problem-specific tricks or guidance that can materially help solve the problem. This structure enables controlled experiments such as adding helpful hints, injecting misleading concepts, or testing whether models can benefit from the “right” side information. CHAMP is also used to study solution faithfulness and verification, not just final-answer accuracy.
    
    \paragraph{TheoremQA}
        TheoremQA is a theorem-driven question answering benchmark that tests whether models can apply domain knowledge (i.e., specific theorems) to solve challenging problems. It consists of expert-curated questions covering hundreds of theorems across multiple disciplines, including mathematics, physics, EE/CS, and finance. Problems are designed so that solving them requires identifying and correctly using an appropriate theorem rather than relying on surface patterns. The dataset emphasizes reasoning and knowledge application: models must connect the question to relevant formal principles and carry out multi-step derivations to reach the final answer. Evaluation commonly focuses on final-answer accuracy (often exact match), with optional analysis of reasoning quality.
    
    \paragraph{MMLU-Redux}
        MMLU-Redux is a cleaned, manually re-annotated variant of the widely used MMLU benchmark, aiming to correct ambiguities and label errors that can distort model comparisons. It preserves the broad subject coverage of MMLU (spanning dozens of academic and professional domains) while providing revised gold answers and improved quality control. The benchmark is intended for more reliable measurement of factual recall and reasoning in a multiple-choice setting, and it is particularly useful for checking whether model rankings are robust to dataset noise. In practice, evaluation remains standard multiple-choice accuracy, but conclusions are less sensitive to flawed items.
    
    \paragraph{GPQA}
        GPQA is a graduate-level, “Google-proof” multiple-choice QA benchmark written by domain experts to be exceptionally difficult even with unrestricted web access. It focuses on deep scientific reasoning in three core areas: biology, chemistry, and physics. The questions are designed to resist shallow retrieval and require substantial conceptual understanding and careful elimination of distractors. Because humans without specialized expertise struggle on the dataset, GPQA is often used to stress-test frontier models and to study oversight and evaluation reliability on hard scientific questions. Performance is typically reported as multiple-choice accuracy, sometimes with analysis by domain and difficulty.
    
    \paragraph{NaturalQuestions}
        NaturalQuestions (NQ) is a large-scale QA dataset built from real, anonymized Google search queries paired with Wikipedia pages from top search results. For each question-page pair, human annotators label a long answer (usually a paragraph or section) and a short answer (one or more spans/entities), or mark the question as unanswerable from the page. This design supports both reading comprehension and retrieval-style evaluation, since the system must identify evidence within long documents. NQ includes diverse information-seeking questions and naturally occurring ambiguity, making it a strong benchmark for open-domain QA pipelines. Common metrics include exact match and token-level F1 for short answers, plus long-answer accuracy.
    
    \paragraph{FinQA}
        FinQA is a financial question answering dataset targeting numerical reasoning over real corporate reports. Each instance pairs a question with evidence drawn from heterogeneous sources such as tables and accompanying textual context. Solving typically requires multi-step arithmetic, aggregation, and logical operations (e.g., computing growth rates, margins, or differences across periods). A key feature of FinQA is the availability of gold “reasoning programs” (structured operation sequences) that make the computation process explicit, enabling explainable evaluation and program-based training or analysis. Systems are commonly evaluated by final numeric answer accuracy (often exact match after normalization), and sometimes by program correctness or execution validity.

\section{Experimental Setup}\label{app:experiment_setup}
    \begin{table}[t]
        \centering
        \small
        \caption{
            The prompt of inference during our experiments.
            \{information\} can be replaced with demonstrations, retrieved knowledge, and memory.
        }
        \begin{tabularx}{\linewidth}{@{}>{\raggedright\arraybackslash}X@{}}
    \toprule
    Prompt of Inference\\
    \midrule
    \texttt{\{task\_definition\}}\newline
    \newline
    \texttt{---}\newline
    \newline
    \texttt{Here are some information to help you solve the given task:}\newline
    \texttt{```}\newline
    \texttt{\{information\}}\newline
    \texttt{```}\newline
    \newline
    \texttt{---}\newline
    \newline
    \texttt{Based on the above information, solve the following question:}\newline
    \texttt{\{question\}}\\
    \bottomrule
\end{tabularx}

        \label{tab:prompt}
    \end{table}
    \begin{table*}[t]
        \centering
        \small
        \caption{
            The prompt of Thinker during memory evolution.
        }
        \begin{tabularx}{\linewidth}{@{}>{\raggedright\arraybackslash}X@{}}
    \toprule
    Prompt of Thinker using Memory Evolution\\
    \midrule
    \texttt{You are an expert meta-reasoning coach for AI agents.}\newline
    \newline
    \texttt{You will receive a SINGLE data sample from an agent’s past experience.}\newline
    \texttt{This sample may include: a task description, the agent’s thoughts/actions/observations,}\newline
    \texttt{and whether the attempt succeeded or failed.}\newline
    \newline
    \texttt{Your job is to distill this ONE example into ONE HIGH-LEVEL INSIGHT:}\newline
    \texttt{a general rule or strategy that would be useful for similar future tasks.}\newline
    \newline
    \texttt{================ INPUT SAMPLE ================}\newline
    \texttt{[Task description]}\newline
    \texttt{\{TASK\_DESCRIPTION\}}\newline
    \newline
    \texttt{[Episode log / trajectory]}\newline
    \texttt{\{EPISODE\_LOG\}}\newline
    \texttt{=============================================}\newline
    \newline
    \texttt{Please now produce exactly ONE INSIGHT that:}\newline
    \newline
    \texttt{- Is HIGH-LEVEL and GENERAL.}\newline
    \texttt{  - It should be a reusable principle that could help on future, similar tasks.}\newline
    \texttt{  - It should focus on better thinking and acting (e.g., how to search, plan, verify, explore).}\newline
    \newline
    \texttt{- Is NOT tied to superficial details of this specific sample.}\newline
    \texttt{  - Do NOT mention specific names, URLs, IDs, or exact numbers from the log}\newline
    \texttt{    unless they are essential to the principle.}\newline
    \texttt{  - Do NOT describe the episode step-by-step.}\newline
    \texttt{  - Do NOT say “In this trial…” or “In this example…”.}\newline
    \texttt{    Instead, phrase it as “When solving tasks like this, you should…”.}\newline
    \newline
    \texttt{- Captures what the agent should DO DIFFERENTLY or KEEP DOING.}\newline
    \texttt{  - If the attempt failed: focus on what change of strategy would most likely prevent this failure.}\newline
    \texttt{  - If the attempt succeeded: focus on what strategy made it robust and should be repeated.}\newline
    \newline
    \texttt{Output format (VERY IMPORTANT):}\newline
    \newline
    \texttt{INSIGHT:}\newline
    \texttt{- <one single sentence (or two short sentences) stating the general rule>}\newline
    \newline
    \texttt{Do NOT output anything else. No explanations, no justification, no bullet list of multiple rules.}\newline
    \texttt{Just the single “INSIGHT:” line as specified above.}\\
    \bottomrule
\end{tabularx}

        \label{tab:prompt_me}
    \end{table*}

    In this paper, for all three types of external information, we use BM25 to retrieve information relevant to the user query.

    The inference prompt we used is shown in Table~\ref{tab:prompt}.
    Specifically, our method designs for different types of information are as follows:
    \paragraph{In-Context Learning}
        We use a sample ICL approach, where the input consists of the task definition, demonstrations, and the user query, and the output is the answer to the query.
        The demonstration pool is the training set for each dataset.
    \paragraph{Retrieval-Augmented Generation}
        We directly use the retrieval corpus provided by each dataset and simply concatenate the retrieved results into the input.
    \paragraph{Memory Evolution}
        Following \citep{zhao2024expel}, we implement evolution with a simple agent system consisting of two components: a Thinker and a Reasoner.
        The Thinker learns from historical interactions and summarizes experiences, while the Reasoner answers user queries based on the summarized experience.
        The Thinker prompt we used is shown in Table~\ref{tab:prompt_me}.

\section{Calculation of Main Parameters}\label{app:calculation}
    In this section, we introduce how to compute each analysis parameter in \S\ref{sec:experiment}.
    To reduce the experimental cost, we calculate the following parameters with the first layer of the model, as we consider the first layer as a single Transformer block.
    \begin{itemize}[leftmargin=*]
        \item $e(x)$ and $e(t,x)$: We directly compute the error as the difference between the pooling of the encoding vectors of each output token and the corresponding token in the ground-truth answer, under the inputs $x$ and $[t,x]$, respectively.
        \item $g(t,x)$: Our computation follows the definition of $g(t,x)$, where $\mathrm{Attn}_x([x])$ and $\delta_{\mathrm{attn}}(t,x)$ can be obtained directly from the model parameters. For $J_{\sigma}$ in $J_{\mathrm{MLP}}$, we derive the Jacobian matrices of the activation functions for each LLM accordingly.
        \item $\rho(t,x)$: Based on the definition in Corollary~\ref{cor:proper_correct}, we can compute it directly from $e(x)$ and $g(t,x)$.
        \item $\alpha_{x \leftarrow t}$: We directly extract the attention weights from $x$ to $t$ when feeding $[t,x]$ into the model.
    \end{itemize}

\section{Effect of Other Structure of Transformer}\label{app:other_structure_of_transformer}
    In this section, we analyze how different factors not discussed in \S\ref{sec:analysis} affect the conclusion of this paper.
    \paragraph{Multi-Head Attention.}
        Multi head attention keeps the same decomposition view because the model still produces an attention output at position $x$, and the information-corrective vector is simply the composition of the per head contributions through the output projection. \textbf{Overall, MHA does not change our main conclusion.}
    \paragraph{Attention Variants.}
        Sharing keys and values across heads mainly redistributes where relevance and complementarity reside across heads, while the information-corrective vector is still governed by how much attention is assigned to the context and how informative the context values are in the projected space. \textbf{In short, these variants preserve the key conclusion.}
    \paragraph{Feed Forward Network Variants.}
        Gating and alternative nonlinearities may alter the local amplification or attenuation of an update, but error reduction still hinges on whether the induced vector aligns with the negative original error and whether its magnitude is not excessively large. \textbf{Thus, the main conclusion remains unchanged.}
    \paragraph{Positional Encoding and Attention Bias.}
        These mechanisms mostly reshape attention weights and, therefore, affect relevance.
        The same relationship between vector direction, vector norm, and the possibility of error reduction continues to apply. \textbf{Consequently, our central conclusion still holds.}
    \paragraph{Masking and Long Context Visibility Constraints.}
        Visibility constraints decide whether information tokens can contribute at all, but whenever they are visible, their impact on the error can still be interpreted through the same vector based lens. \textbf{So they do not overturn the main conclusion either.}

\section{Performance of Trained Predictor}\label{app:performance_of_trained_predictor}
    \begin{table}[t]
        \centering
        \small
        \caption{
            The relative error ($\frac{|\texttt{gold\_value - pred\_value}|}{\texttt{gold\_value}}$) of the trained predictor under different settings.
        }
        \begin{tabular}{ll|cc}
    \toprule
    \textbf{Dataset} & \textbf{Method} & \textbf{Llama-3.1-8B} & \textbf{Qwen3-8B} \\
    \midrule
    \multirow{2}{*}{MATH} & ICL & $33.8\%$ & $39.4\%$ \\
     & Mem & $38.3\%$ & $25.5\%$ \\
    \midrule
    \multirow{2}{*}{CHAMP} & RAG & $27.0\%$ & $23.9\%$ \\
     & Mem & $34.4\%$ & $36.7\%$ \\
    \midrule
    \multirow{2}{*}{TheoremQA} & RAG & $26.6\%$ & $26.0\%$ \\
     & Mem & $28.8\%$ & $29.5\%$ \\
    \midrule
    MMLU-Redux & Mem & $31.4\%$ & $33.4\%$ \\
    \midrule
    \multirow{2}{*}{GPQA} & ICL & $34.4\%$ & $21.1\%$ \\
     & Mem & $37.4\%$ & $25.5\%$ \\
    \midrule
    NQ & Mem & $39.3\%$ & $33.5\%$ \\
    \midrule
    \multirow{2}{*}{FinQA} & ICL & $25.2\%$ & $38.1\%$ \\
     & Mem & $38.1\%$ & $27.1\%$ \\
    \bottomrule
\end{tabular}

        \label{tab:performance_of_trained_predictor}
    \end{table}

    The relative error of the trained predictor is shown in Table~\ref{tab:performance_of_trained_predictor}.

\section{Effects of Different Factors on Method Performance}\label{app:method_performance_factors}
    This section studies factors that affect the performance of the theory-guided information-selection method in Section~\ref{sec:selection}. We examine the generalization and accuracy of the learned correction-direction predictor, the contribution of the two theory-derived selection components, and comparisons with stronger selection baselines under the same candidate pool and computational budget.
    
    \subsection{Effects of Predictor Generalization and Accuracy}\label{app:predictor_generalization}
        Section~\ref{sec:selection} uses a learned predictor for the correction-error alignment score $\rho(t,x)$ because the ground-truth answer is unavailable at inference time. We further examine whether this predictor transfers across models and datasets and quantify the gap between predicted and oracle correction-direction filtering.
        
        \paragraph{Cross-model transfer.}
            \begin{table}[t]
                \centering
                \small
                \caption{Cross-model transfer of the learned correction-direction predictor. Values are average downstream accuracy (\%) for each external-information setting.}
                \begin{tabular}{ll|cccc}
    \toprule
    \textbf{Predictor Source} & \textbf{Target Model} & \textbf{ICL} & \textbf{RAG} & \textbf{Mem} & \textbf{Avg.} \\
    \midrule
    Llama-3.1-8B & Llama-3.1-8B & 49.7 & 49.8 & 42.0 & 45.6 \\
    Qwen3-8B     & Llama-3.1-8B & 48.8 & 48.6 & 40.9 & 44.5 \\
    \midrule
    Qwen3-8B     & Qwen3-8B     & 75.1 & 64.6 & 62.9 & 66.1 \\
    Llama-3.1-8B & Qwen3-8B     & 74.5 & 63.3 & 62.0 & 65.2 \\
    \bottomrule
\end{tabular}

                \label{tab:predictor_cross_model}
            \end{table}

            Table~\ref{tab:predictor_cross_model} reports downstream accuracy when the predictor trained for one model is applied to another target model. Cross-model transfer reduces downstream accuracy by $0.6$--$1.3$ percentage points across the evaluated settings.

        \paragraph{Cross-dataset transfer.}
            \begin{table}[t]
                \centering
                \small
                \caption{Cross-dataset generalization of the learned correction-direction predictor. Values are downstream accuracy (\%). ``Other Datasets'' denotes the cross-dataset training setting.}
                \begin{tabular}{l|cc}
    \toprule
    \textbf{Target Dataset} & \textbf{In-Domain} & \textbf{Other Datasets} \\
    \midrule
    MATH       & 69.6 & 68.3 \\
    CHAMP      & 34.1 & 33.1 \\
    TQA        & 49.7 & 49.0 \\
    MMLU-Redux & 68.5 & 68.0 \\
    GPQA       & 53.9 & 53.0 \\
    NQ         & 61.9 & 60.5 \\
    FinQA      & 59.5 & 58.6 \\
    \bottomrule
\end{tabular}

                \label{tab:predictor_cross_dataset}
            \end{table}
            
            Table~\ref{tab:predictor_cross_dataset} compares the in-domain setting with the cross-dataset setting. Cross-dataset transfer reduces accuracy by $0.5$--$1.4$ percentage points across the seven benchmarks.
        
        \paragraph{Predicted vs. oracle filtering.}
            \begin{table}[t]
                \centering
                \small
                \caption{Effect of correction-direction estimation on downstream accuracy (\%).}
                \begin{tabular}{ll|ccc}
    \toprule
    \textbf{Model} & \textbf{Setting} & \textbf{No-filter} & \textbf{Predicted} & \textbf{Oracle} \\
    \midrule
    \multirow{4}{*}{Llama-3.1-8B}
    & ICL & 48.6 & 49.7 & \textbf{54.0} \\
    & RAG & 48.1 & 49.8 & \textbf{55.2} \\
    & Mem & 40.8 & 42.0 & \textbf{46.3} \\
    & \textbf{Average} & 44.3 & 45.6 & \textbf{50.1} \\
    \midrule
    \multirow{4}{*}{Qwen3-8B}
    & ICL & 74.2 & 75.1 & \textbf{79.0} \\
    & RAG & 63.2 & 64.6 & \textbf{69.8} \\
    & Mem & 62.2 & 62.9 & \textbf{67.4} \\
    & \textbf{Average} & 65.2 & 66.1 & \textbf{70.7} \\
    \bottomrule
\end{tabular}

                \label{tab:predictor_oracle}
            \end{table}
        
            Table~\ref{tab:predictor_oracle} compares three variants: \textit{No-filter} removes the $\rho(t,x)$-based direction filter, \textit{Predicted} uses the learned predictor, and \textit{Oracle} uses the ground-truth $\rho(t,x)$. Predicted filtering improves upon removing the direction filter, and oracle filtering yields the highest accuracy in each setting.

    \subsection{Effects of the Selection Components}\label{app:selection_ablation}
        \begin{table}[t]
            \centering
            \small
            \caption{Ablation of the two selection components in ICL. Values are accuracy (\%).}
            \begin{tabular}{ll|ccc}
    \toprule
    \textbf{Model} & \textbf{Variant} & \textbf{MATH} & \textbf{GPQA} & \textbf{FinQA} \\
    \midrule
    \multirow{3}{*}{Llama-3.1-8B}
    & Direction Filter Only & 49.2 & 45.6 & 52.7 \\
    & Strength Proxy Only   & 48.5 & 44.3 & 52.2 \\
    & Both                  & \textbf{49.6} & \textbf{46.5} & \textbf{53.1} \\
    \midrule
    \multirow{3}{*}{Qwen3-8B}
    & Direction Filter Only & 89.6 & 65.5 & 69.0 \\
    & Strength Proxy Only   & 89.0 & 65.3 & 68.2 \\
    & Both                  & \textbf{90.2} & \textbf{65.7} & \textbf{69.5} \\
    \bottomrule
\end{tabular}

            \label{tab:ablation_icl}
        \end{table}
        \begin{table}[t]
            \centering
            \small
            \caption{Ablation of the two selection components in RAG. Values are accuracy (\%).}
            \begin{tabular}{ll|ccc}
    \toprule
    \textbf{Model} & \textbf{Variant} & \textbf{CHAMP} & \textbf{TQA} & \textbf{NQ} \\
    \midrule
    \multirow{3}{*}{Llama-3.1-8B}
    & Direction Filter Only & 26.6 & 41.9 & 78.9 \\
    & Strength Proxy Only   & 25.9 & 40.9 & 77.8 \\
    & Both                  & \textbf{27.4} & \textbf{42.5} & \textbf{79.4} \\
    \midrule
    \multirow{3}{*}{Qwen3-8B}
    & Direction Filter Only & 47.0 & 66.7 & 78.2 \\
    & Strength Proxy Only   & 45.8 & 65.5 & 76.2 \\
    & Both                  & \textbf{47.8} & \textbf{67.2} & \textbf{78.9} \\
    \bottomrule
\end{tabular}

            \label{tab:ablation_rag}
        \end{table}
        \begin{table*}[t]
            \centering
            \small
            \caption{Ablation of the two selection components in the memory setting. Values are accuracy (\%).}
            \resizebox{\linewidth}{!}{
\begin{tabular}{ll|ccccccc}
    \toprule
    \textbf{Model} & \textbf{Variant} & \textbf{MATH} & \textbf{CHAMP} & \textbf{TQA} & \textbf{MMLU} & \textbf{GPQA} & \textbf{NQ} & \textbf{FinQA} \\
    \midrule
    \multirow{3}{*}{Llama-3.1-8B}
    & Direction Filter Only & 48.9 & 17.6 & 29.8 & 66.5 & 34.3 & 41.1 & 51.4 \\
    & Strength Proxy Only   & 47.9 & 15.8 & 28.2 & 65.8 & 32.8 & 40.3 & 50.5 \\
    & Both                  & \textbf{49.6} & \textbf{18.1} & \textbf{30.4} & \textbf{67.1} & \textbf{34.8} & \textbf{41.7} & \textbf{52.1} \\
    \midrule
    \multirow{3}{*}{Qwen3-8B}
    & Direction Filter Only & 88.7 & 42.2 & 57.8 & 69.1 & 67.6 & 47.1 & 62.7 \\
    & Strength Proxy Only   & 88.5 & 41.3 & 56.5 & 67.4 & 66.7 & 45.8 & 61.9 \\
    & Both                  & \textbf{89.0} & \textbf{43.0} & \textbf{58.7} & \textbf{69.9} & \textbf{68.7} & \textbf{47.5} & \textbf{63.4} \\
    \bottomrule
\end{tabular}
}

            \label{tab:ablation_mem}
        \end{table*}
    
        Our selection strategy contains two theory-derived components: a correction-direction filter based on the predicted $\rho(t,x)$ and a corrective-strength proxy $\alpha_{x\leftarrow t}\|v_t-v_x\|$. We ablate the two components to quantify their individual contributions. In all three external-information settings, using both components consistently outperforms using either component alone. The direction filter is also generally stronger than the strength proxy alone, consistent with the observation in Section~\ref{subsec:distribution_pattern} that incorrect correction direction is a major source of harmful external information.
    
    \subsection{Comparison with Stronger Selection Baselines}\label{app:stronger_baselines}
        \begin{table}[t]
            \centering
            \small
            \caption{Comparison with stronger ICL/context-selection baselines under the same candidate pool and computational budget. Values are accuracy (\%).}
            \begin{tabular}{ll|ccc}
    \toprule
    \textbf{Model} & \textbf{Method} & \textbf{MATH} & \textbf{GPQA} & \textbf{FinQA} \\
    \midrule
    \multirow{4}{*}{Llama-3.1-8B}
    & Linear-Time DS & 48.8 & 44.4 & 52.4 \\
    & AutoSelect     & 49.1 & 45.9 & 52.8 \\
    & DiSP           & 49.3 & 46.0 & 52.7 \\
    & \textbf{Ours}  & \textbf{49.6} & \textbf{46.5} & \textbf{53.1} \\
    \midrule
    \multirow{4}{*}{Qwen3-8B}
    & Linear-Time DS & 89.3 & 65.5 & 67.8 \\
    & AutoSelect     & 89.8 & 65.6 & 69.3 \\
    & DiSP           & \textbf{90.2} & 65.6 & 68.8 \\
    & \textbf{Ours}  & \textbf{90.2} & \textbf{65.7} & \textbf{69.5} \\
    \bottomrule
\end{tabular}

            \label{tab:strong_baselines_icl}
        \end{table}
        \begin{table}[t]
            \centering
            \small
            \caption{Comparison with stronger RAG/context-selection baselines under the same candidate pool and computational budget. Values are accuracy (\%).}
            \begin{tabular}{ll|ccc}
    \toprule
    \textbf{Model} & \textbf{Method} & \textbf{CHAMP} & \textbf{TQA} & \textbf{NQ} \\
    \midrule
    \multirow{4}{*}{Llama-3.1-8B}
    & Qwen3-Embedding-0.6B & 25.9 & 41.2 & 78.1 \\
    & Qwen3-Reranker-0.6B  & 26.6 & 41.6 & 78.7 \\
    & jina-reranker-v3     & 26.8 & 42.0 & 78.9 \\
    & \textbf{Ours}        & \textbf{27.4} & \textbf{42.5} & \textbf{79.4} \\
    \midrule
    \multirow{4}{*}{Qwen3-8B}
    & Qwen3-Embedding-0.6B & 46.7 & 66.2 & 77.5 \\
    & Qwen3-Reranker-0.6B  & 47.2 & 66.5 & 78.1 \\
    & jina-reranker-v3     & 47.4 & 66.5 & 78.5 \\
    & \textbf{Ours}        & \textbf{47.8} & \textbf{67.2} & \textbf{78.9} \\
    \bottomrule
\end{tabular}

            \label{tab:strong_baselines_rag}
        \end{table}
        \begin{table*}[t]
            \centering
            \small
            \caption{Comparison with stronger memory-selection baselines under the same candidate pool and computational budget. Values are accuracy (\%).}
            \resizebox{\linewidth}{!}{
\begin{tabular}{ll|ccccccc}
    \toprule
    \textbf{Model} & \textbf{Method} & \textbf{MATH} & \textbf{CHAMP} & \textbf{TQA} & \textbf{MMLU} & \textbf{GPQA} & \textbf{NQ} & \textbf{FinQA} \\
    \midrule
    \multirow{4}{*}{Llama-3.1-8B}
    & MemOS      & 48.6 & 16.6 & 28.7 & 65.6 & 33.7 & 40.7 & 51.0 \\
    & SimpleMem  & 48.9 & 17.2 & 29.6 & 66.5 & 34.0 & 41.6 & 51.8 \\
    & MemGuide   & 49.4 & 17.7 & 29.4 & 66.8 & 34.0 & 41.2 & 51.7 \\
    & \textbf{Ours} & \textbf{49.6} & \textbf{18.1} & \textbf{30.4} & \textbf{67.1} & \textbf{34.8} & \textbf{41.7} & \textbf{52.1} \\
    \midrule
    \multirow{4}{*}{Qwen3-8B}
    & MemOS      & 88.5 & 41.5 & 57.2 & 68.4 & 67.4 & 46.5 & 62.7 \\
    & SimpleMem  & 88.9 & 42.3 & 57.6 & 69.2 & 67.8 & 46.8 & 62.3 \\
    & MemGuide   & \textbf{89.1} & 42.7 & 58.4 & 69.8 & 67.8 & 47.1 & 62.8 \\
    & \textbf{Ours} & 89.0 & \textbf{43.0} & \textbf{58.7} & \textbf{69.9} & \textbf{68.7} & \textbf{47.5} & \textbf{63.4} \\
    \bottomrule
\end{tabular}
}

            \label{tab:strong_baselines_mem}
        \end{table*}
        
        We further compare the theory-guided selector with stronger and more recent context-selection baselines under the same candidate pool and computational budget. As in Section~\ref{sec:selection}, each query starts from $32$ coarsely retrieved candidates and all methods ultimately select top-$1$ external information for inference. Tables~\ref{tab:strong_baselines_icl}-\ref{tab:strong_baselines_mem} report the results for ICL, RAG, and memory-based settings, respectively. Our method achieves the highest or tied-highest accuracy in $25$ of the $26$ model--benchmark combinations.



\end{document}